\documentclass{article}

\PassOptionsToPackage{numbers, sort&compress}{natbib}
\usepackage[final]{nips_2017}

\usepackage[utf8]{inputenc} % allow utf-8 input
\usepackage[T1]{fontenc}    % use 8-bit T1 fonts
\usepackage[pagebackref]{hyperref}       % hyperlinks
\usepackage{url}            % simple URL typesetting
\usepackage{booktabs}       % professional-quality tables
\usepackage{amsfonts}       % blackboard math symbols
\usepackage{nicefrac}       % compact symbols for 1/2, etc.
\usepackage{microtype}      % microtypography
\usepackage{algpseudocode}
\usepackage{amsthm,amssymb,amsopn,amsmath}
\usepackage{graphicx} % more modern
\usepackage{graphicx}
\usepackage{subcaption}
\usepackage{wrapfig}

\newtheorem{define}{Definition}

\newtheorem{lem}{Lemma}

\title{Counterfactual Fairness}
\usepackage{listings}
\usepackage{color}

\newcommand*\samethanks[1][\value{footnote}]{\footnotemark[#1]}

\usepackage{listings}
\usepackage{color}

\definecolor{dkgreen}{rgb}{0,0.6,0}
\definecolor{gray}{rgb}{0.5,0.5,0.5}
\definecolor{mauve}{rgb}{0.58,0,0.82}

\author{
  Matt Kusner \thanks{Equal contribution. This work was done while JL was a Research Fellow at the Alan Turing Institute.}\\
  The Alan Turing Institute and \\
  University of Warwick \\
  \texttt{mkusner@turing.ac.uk} \\
  \And
  Joshua Loftus \samethanks\\
  New York University \\
  \texttt{loftus@nyu.edu} \\
  \And
  Chris Russell \samethanks \\
  The Alan Turing Institute and \\
  University of Surrey\\
  \texttt{crussell@turing.ac.uk} \\
  \And
  Ricardo Silva \\
  The Alan Turing Institute and \\
  University College London\\
  \texttt{ricardo@stats.ucl.ac.uk} \\
}

\begin{document}
\maketitle

% this must go after the closing bracket ] following \twocolumn[ ...

% This command actually creates the footnote in the first column
% listing the affiliations and the copyright notice.
% The command takes one argument, which is text to display at the start of the footnote.
% The \icmlEqualContribution command is standard text for equal contribution.
% Remove it (just {}) if you do not need this facility.

%\printAffiliationsAndNotice{}  % leave blank if no need to mention equal contribution
%\printAffiliationsAndNotice{\textsuperscript{*}Equal contribution, author order decided randomly}
%\icmlEqualContribution}
%%\icmlEqualContribution} % otherwise use the standard text.
%\footnotetext{hi}

\begin{abstract} 
% fairness is important
  Machine learning can impact people with legal or ethical
  consequences when it is used
  %has matured to the point to where it is now being   considered
  to automate decisions in areas such as insurance, lending, hiring,
  and predictive policing.  In many of these scenarios,
  previous decisions have been made that are unfairly biased against
  certain subpopulations, for example those of a particular race, gender, or
  sexual orientation.  Since this past data may be biased,
  machine learning predictors must account for this to avoid
  perpetuating or creating discriminatory practices.
  In this paper, we develop a framework for modeling fairness
  using tools from causal inference. Our definition of
  \emph{counterfactual fairness} captures the
  intuition that a decision is fair towards an individual if it is
  the same in (a) the actual world and (b) a counterfactual world
  where the individual belonged to a different demographic
  group. We demonstrate our framework on a real-world problem of fair
  prediction of success in law school.
  %demonstrate our framework on two real-world problems: fair
  %prediction of law school success, and fair modeling of an
  %individual's criminality in policing data.
\end{abstract} 

\section{Contribution}
\label{sec:introduction}
%!TEX root=ricardo_draft.tex
% ml is now everywhere
Machine learning has spread to fields as diverse as credit scoring
\cite{khandani2010consumer}, crime prediction
\cite{brennan2009evaluating}, and loan assessment
\cite{mahoney2007method}. Decisions in these areas may have ethical or
legal implications, so it is necessary for the modeler to think beyond
the objective of maximizing prediction accuracy and consider the
societal impact of their work.
% in these new ml fields, we cannot discriminate
% discrimination can happen in multiple ways
% - direct discrimination
For many of these applications, it is crucial to ask if the
predictions of a model are \emph{fair}.  Training data can contain
unfairness for reasons having to do with historical prejudices or
other factors outside an individual's control.
% For instance, imagine a bank
% wishes to predict if an individual should be given a loan to buy a
% house. The bank wishes to use historical repayment data, alongside
% individual data. If they simply learn a model that predicts whether
% the loan will be paid back, it may unjustly favor applicants of
% particular subgroups, due to past and present prejudices.
In 2016, the Obama administration released a
report\footnote{https://obamawhitehouse.archives.gov/blog/2016/05/04/big-risks-big-opportunities-intersection-big-data-and-civil-rights}
which urged data scientists to analyze ``how technologies can
deliberately or inadvertently perpetuate, exacerbate, or mask
discrimination."

There has been much recent interest in designing algorithms that make
fair predictions
\cite{hardt2016equality,dwork2012fairness,joseph2016rawlsian,kamishima2011fairness,zliobaite2015survey,zafar2016fairness,zafar2015learning,grgiccase,kleinberg:17,calders2010three,kamiran2012data,bolukbasi2016man,kamiran2009classifying,zemel2013learning,louizos2015variational}.
% Most
% of these works focus on formalizing fairness into a numeric
% definition and satisfying it with customized algorithms.
In large part, the literature has focused on formalizing fairness into
quantitative definitions and using them to solve a discrimination
problem in a certain dataset. Unfortunately, for a practitioner,
law-maker, judge, or anyone else who is interested in implementing
algorithms that control for discrimination, it can be difficult to
decide {\em which} definition of fairness to choose for the task at
hand. Indeed, we demonstrate that depending on the relationship
between a protected attribute and the data, certain definitions of
fairness can actually \emph{increase discrimination}.

% we propose a way to model data that allows a practitioner to assess what definitions of fairness are right for the problem at hand, and algorithms to ensure fairness
% OR
% we propose a way to interpret fairness...
% a) relationship between fairness and causality
% b) use pearl's models
% c) having an explicit model allows us to test fairness with the assumptions laid bare
% tension: Pearl already talks about discrimination, so we aren't really inventing new models. Are we even new in using these models to talk about fairness? Maybe... Pearl talks about variables that we might want to compute counterfactuals for in order to see if discrimination is happening.  
% Our proposal is:
% - situate a sensitive variable in a graph (not new).
% - Look at old definitions and see if anything bad could happen (new). 
% - Then define counterfactual fairness (new). 
% - Modeling helps us see where the weaknesses are in our assumptions and definitions (maybe not new)
% We don't want to see if every definition is counterfactually fair because then we're like everyone else, saying our definition is best
% 
% In this work,

In this paper, we introduce the first explicitly causal approach to
address fairness.  Specifically, we leverage the causal framework of
\citet{pearl2009causal} to model the relationship between protected
attributes and data. We describe how techniques from causal inference
can be effective tools for designing fair algorithms and argue, as in
\citet{dedeo2014wrong}, that it is essential to properly address
causality in fairness. In perhaps the most closely related prior work,
\citet{johnson2016impartial} make similar arguments but from a
non-causal perspective. An alternative use of causal modeling in
the context of fairness is introduced independently by \citep{kilbertus:17}.

In Section \ref{sec:background}, we provide a summary of basic
concepts in fairness and causal modeling. In Section
\ref{sec:count_fair}, we provide the formal definition of
\emph{counterfactual fairness}, which enforces that a distribution
over possible predictions for an individual should remain unchanged in
a world where an individual's protected attributes had been different
in a causal sense. In Section \ref{sec:methods}, we describe an
algorithm to implement this definition, while distinguishing it from
existing approaches.  In Section \ref{sec:experiments}, we illustrate
the algorithm with a case of fair assessment of law school success.

\section{Background}
\label{sec:background}
%!TEX root=ricardo_draft.tex
% We begin by describing the problem of fair prediction and introduce three of the most popular definitions developed for this task.  We then give a brief overview of causal modeling which will act as our `tool-kit' for modeling and defining fairness.

% !TEX root=ricardo_draft.tex
% binary classification
% In essence, the primary challenge in training fair classifiers comes
% from the world itself  not being fair, and as such an
% unbiased (in the technical sense) classifier trained from real-world
% data is not guaranteed to be fair (in the social sense).

This section provides a basic account of two separate areas of
research in machine learning, which are formally unified in this
paper. We suggest \citet{berk:17} and \citet{pearl:16} as references.
Throughout this paper, we will use the following notation.  Let $A$
denote the set of {\it protected attributes} of an individual,
variables that must not be discriminated against in a formal sense
defined differently by each notion of fairness discussed. The decision
of whether an attribute is protected or not is taken as a primitive in
any given problem, regardless of the definition of fairness
adopted. Moreover, let $X$ denote the other observable attributes of
any particular individual, $U$ the set of relevant latent attributes
which are not observed, and let $Y$ denote the outcome to be
predicted, which itself might be contaminated with historical
biases. Finally, $\hat Y$ is the {\it predictor}, a random variable
that depends on $A, X$ and $U$, and which is produced by a machine
learning algorithm as a prediction of $Y$.

\subsection{Fairness}

There has been much recent work on fair algorithms.  These include
fairness through unawareness \cite{grgiccase}, individual fairness
\cite{dwork2012fairness,zemel2013learning,louizos2015variational,
  joseph2016rawlsian}, demographic parity/disparate impact
\cite{zafar2015learning}, and equality of opportunity
\cite{hardt2016equality,zafar2016fairness}.  For simplicity we often
assume $A$ is encoded as a binary attribute, but this can be
generalized.

% A variety of algorithmic approaches to achieving fairness have been proposed.
%\begin{description}

\begin{define}[Fairness Through Unawareness (FTU)]
  An algorithm is fair so long as any protected attributes $A$ are not
  explicitly used in the decision-making process. 
  %(or other unfair attributes, see \citet{grgiccase}) satisfies this.
\end{define}
Any mapping $\hat{Y}: X \rightarrow Y$ that excludes $A$ satisfies
this. Initially proposed as a baseline, the approach has found
favor recently with more general approaches such as \citet{grgiccase}.
Despite its compelling simplicity, FTU has a clear
shortcoming as elements of $X$ can contain discriminatory information
analogous to $A$ that may not be obvious at first. The need for expert
knowledge in assessing the relationship between $A$ and $X$ was
highlighted in the work on individual fairness:
%and constructs a predictor $\hat
%Y$ based on a feature vector $X$ that excludes $A$, and in the case of
%\citet{grgiccase} other attributes labeled as unfair.
%
\begin{define}[Individual Fairness (IF)]
  An algorithm is fair if it gives similar predictions to similar
  individuals. Formally, given a metric $d(\cdot,\cdot)$, if individuals $i$ and $j$ are similar under this metric (i.e., $d(i,j)$ is small) then their predictions should be similar: $\hat{Y}(X^{(i)}, A^{(i)}) \approx \hat{Y}(X^{(j)}, A^{(j)})$.
   % with features $X^{(i)},X^{(j)}$ and protected attributes $A^{(i)},A^{(j)}$ are similar apart
  % from their protected attributes $A_i$, $A_j$ then
  % $\hat{Y}(X^{(i)}, A^{(i)}) \approx \hat{Y}(X^{(j)}, A^{(j)})$.
%\begin{align}
%  \hat{Y}(X^{(i)}, A^{(i)}) \approx \hat{Y}(X^{(j)}, A^{(j)}).\nonumber
%\end{align}
\end{define}
%This approach can be understood loosely as a continuous analog of
%FTU.
As described in \cite{dwork2012fairness}, the metric $d(\cdot,\cdot)$ must be carefully chosen, requiring an understanding of the domain at
hand beyond black-box statistical modeling. This can also be
contrasted against population level criteria such as
%of fairness
%will not correct for the historical biases described above.
%
\begin{define}[Demographic Parity (DP)] 
A predictor  $\hat{Y}$ satisfies demographic parity if
%\begin{align}
$P(\hat{Y} | A = 0) = P(\hat{Y} | A = 1)$. %\nonumber
%\end{align}
\end{define}
 %       
 % PROBLEMS: (a) It allows that we accept qualified applicants for one value of A and unqualified applicants for another value of A (which could arise naturally if we have little training data about one value of A). (b) It can seriously destroy prediction accuracy. (c) It is technically possible to use such a classifier to justify discrimination, if we pick unqualified applicants in one group. (d) Is ignorant of how data was sampled and ignores fairness with respect to subgroups or super-groups.
%
\begin{define}[Equality of Opportunity (EO)]
 %An algorithm is fair if it is equally accurate for each value of the sensitive attribute $A$.
 A predictor $\hat{Y}$ satisfies equality of opportunity if
%\begin{align}
$P(\hat{Y}=1 | A=0,Y=1) = P(\hat{Y}=1 | A=1,Y=1)$. %\nonumber
%\end{align}
\end{define}
These criteria can be incompatible in general, as discussed in
\cite{kleinberg:17, berk:17, chouldechova:17}.  Following the
motivation of IF and \cite{johnson2016impartial}, we propose that knowledge
about relationships between all
attributes should be taken into consideration, even if strong
assumptions are necessary. Moreover, it is not immediately clear for
any of these approaches in which ways historical biases can be
tackled. We approach such issues from an explicit causal modeling
perspective.

\subsection{Causal Models and Counterfactuals}
\label{subsec:cmc}
We follow % the framework of
\citet{pearl:00}, and define a causal
model as a triple $(U, V, F)$ of sets such that
\begin{itemize}
\item $U$ is a set of latent {\bf background} variables,%\footnote{These are
  %sometimes called {\bf exogenous variables}, but the fact that members of $U$
  %might depend on each other is not relevant to what follows.},
  which are factors not caused by any variable in the set $V$ of {\bf observable} variables;
\item $F$ is a set of functions $\{f_1, \dots, f_n\}$, one for each $V_i \in V$, such
that $V_i = f_i(pa_i, U_{pa_i})$, $pa_i \subseteq V \backslash
\{V_i\}$ and $U_{pa_i} \subseteq U$. Such equations are also known as
{\bf structural equations} \citep{bol:89}.
\end{itemize}
%

%The notation ``$pa_i$'' refers to the ``parents'' of $V_i$ and is
%motivated by the assumption that the model factorizes according to a
%directed acyclic graph (DAG). That is, we can define a directed graph
%${\mathcal G}=(U \cup V, \mathcal E )$ where each node is an element
%of $U \cup V$, and each edge from some $Z \subseteq U \cup V$ to $V_i$
%indicates that $Z \in pa_i \cup U_{pa_i}$. By construction, $\mathcal
%G$ is acyclic.

The notation ``$pa_i$'' refers to the ``parents'' of $V_i$ and is
motivated by the assumption that the model factorizes as a directed
graph, here assumed to be a directed acyclic graph (DAG).  The model
is causal in that, given a distribution $P(U)$ over the background
variables $U$, we can derive the distribution of a subset $Z \subseteq
V$ following an {\bf intervention} on $V\setminus Z$.  An 
  intervention on variable $V_i$ is the substitution of equation $V_i
= f_i(pa_i, U_{pa_i})$ with the equation $V_i = v$ for some $v$. This
captures the idea of an agent, external to the system, modifying it by
forcefully assigning value $v$ to $V_i$,
for example as in a randomized experiment.
%This occurs in a randomized
%controlled trials where the value of $V_i$ is overridden by a
%treatment setting it to $v$, a value chosen at random, and thus
%independent of any other causes.% of the
%system. % The do-calculus of \citet{pearl:00} provides a way to identify
% features of interventional distributions %, where possible,
% using only estimates of the joint distribution of $V$ and knowledge of
% the causal DAG.

The specification of $F$ is a strong assumption but allows for the
calculation of {\bf counterfactual} quantities.  In brief, consider
the following counterfactual statement, ``the value of $Y$ if $Z$ had
taken value $z$'', for two observable variables $Z$ and $Y$. By
assumption, the state of any observable variable is fully determined
by the background variables and structural equations. The
counterfactual is modeled as the solution for $Y$ for a given $U = u$
where the equations for $Z$ are replaced with $Z \!=\!  z$.  We denote
it by $Y_{Z \leftarrow z}(u)$ \cite{pearl:00}, and sometimes as $Y_z$
if the context of the notation is clear.

Counterfactual inference, as specified by a causal model $(U, V, F)$
given evidence $W$, is the computation of probabilities $P(Y_{Z
  \leftarrow z}(U)\ |\ W \!=\! w)$, where $W$, $Z$ and $Y$ are subsets
of $V$. Inference proceeds in three steps, as explained in more detail
in Chapter 4 of \citet{pearl:16}: 1. {\bf Abduction}: for a given
prior on $U$, compute the posterior distribution of $U$ given the
evidence $W = w$; 2. {\bf Action}: substitute the equations for $Z$
with the interventional values $z$, resulting in the modified set of
equations $F_z$; 3. {\bf Prediction}: compute the implied distribution
on the remaining elements of $V$ using $F_z$ and the posterior $P(U\ |
W = w)$.

%\begin{enumerate}
%\item Abduction: for a given prior on $U$, compute the posterior
%  distribution of $U$ given % the evidence
%  $W = w$;
%\item Action: substitute the equations for $Z$ with the interventional
%  values $z$, resulting in the modified set of equations $F_z$;
%\item Prediction: compute the implied distribution on the remaining
%  elements of $V$ using $F_z$ and the posterior $P(U\ | W = w)$.
%\end{enumerate}

%%% Local Variables:
%%% mode: latex
%%% TeX-master: "ricardo_draft"
%%% End:

%\section{Background}
%\label{sec:related}
%\input{related}

%\section{Causal Models and Counterfactuals}
%\label{sec:background}
%\input{background}

\section{Counterfactual Fairness}
\label{sec:count_fair}
Given a predictive problem with fairness considerations, where $A$, $X$ and $Y$
represent the protected attributes, remaining attributes, and output of interest respectively,
let us assume that we are given a causal model $(U, V, F)$, where $V \equiv A \cup X$.
We postulate the following criterion for predictors of $Y$.
\begin{define}[Counterfactual fairness]
Predictor $\hat Y$ is {\bf counterfactually fair}
if under any context $X = x$ and $A = a$,
  \label{eq:cf_definition}
\begin{align}
  P(\hat Y_{A \leftarrow a\ }(U) = y\ |\ X = x, A = a)  =%\nonumber\\ 
  P(\hat Y_{A \leftarrow a'}(U) = y\ |\ X = x, A = a), 
\end{align}
for all $y$ and for any value $a'$ attainable by $A$.
\end{define}
%Simply put,

This notion is closely related to {\bf actual causes}
\cite{halpern:16}, or token causality in the sense that, to be fair,
$A$ should not be a cause of $\hat Y$ in any individual instance. In
other words, changing $A$ while holding things which are not causally
dependent on $A$ constant will not change the distribution of $\hat
Y$.
% \footnote{Notice that we always assume counterfactuals to be
%  well-defined by the model. For instance, ``race'' can be taken as a
%  surrogate for ``perceived race.''}
We also emphasize that
counterfactual fairness is an individual-level definition. This is
substantially different from comparing different individuals that happen to
share the same ``treatment'' $A = a$ and coincide on the values of
$X$, as discussed in Section 4.3.1 of \citep{pearl:16} and the
Supplementary Material. Differences between $X_a$ and $X_{a'}$ must be caused
by variations on $A$ only. Notice also that this definition is
agnostic with respect to how good a predictor $\hat Y$ is, which we
discuss in Section \ref{sec:methods}.

\noindent {\bf Relation to individual fairness}. IF is agnostic with
respect to its notion of similarity metric, which is both a strength
(generality) and a weakness (no unified way of defining similarity).
Counterfactuals and similarities are related, as in the classical
notion of distances between ``worlds'' corresponding to different
counterfactuals \cite{lewis:73}. If $\hat Y$ is a
deterministic function of $W \subset A \cup X \cup U$, as in several
of our examples to follow, then IF can be defined by treating equally two
individuals with the same $W$ in a way that is also counterfactually fair.

\noindent {\bf Relation to \citet{pearl:16}}.  In
Example 4.4.4 of \cite{pearl:16}, the authors condition instead on
$X$, $A$, and the observed realization of $\hat Y$, and calculate the
probability of the counterfactual realization $\hat Y_{A \leftarrow
  a'}$ differing from the factual.
%\footnote{The result is an expression
%  called the ``the probability of sufficiency'' for $A$, capturing the
%  notion that switching $A$ to a different value would be sufficient
%  to change $\hat Y$ with some probability.}.
This example conflates the predictor $\hat Y$ with the outcome $Y$, of
which we remain agnostic in our definition but which is used in the
construction of $\hat Y$ as in Section \ref{sec:methods}. Our framing
makes the connection to machine learning more explicit.
% We also emphasize that counterfactual fairness is an individual-level
% definition. This is substantially different from the notion of ``causal independence''
% as discussed in Section 4.3.1 of \cite{pearl:16}. Causal independence
% % requires
% % \begin{align}
% %   &P(\hat Y = y\ |\ do(A = a), W = w) =\nonumber\\ 
% %   &P(\hat Y = y\ |\ do(A = a'), W = w),
% % \end{align}
% % which
% entails comparing different units that happen to share the
% same ``treatment'' and coincide on values of $W$, while
% counterfactual fairness concerns the variation possible within an
% individual depending on their value of $a$ and the descendents of
% $A$ in the causal graph.

\begin{figure}
\centerline{\includegraphics[width=\textwidth]{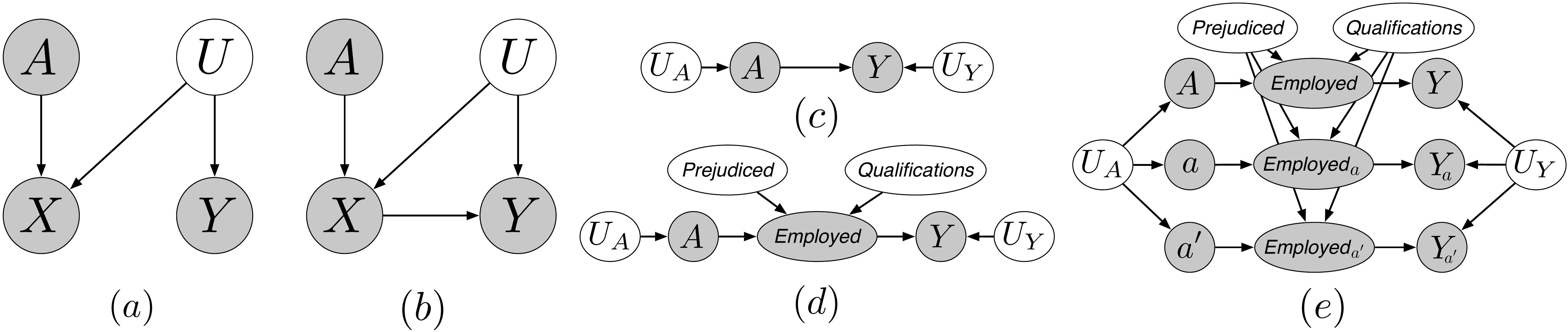}}
\caption{\label{fig:ex1}
    (a), (b) Two causal models for different
    real-world fair prediction scenarios.\label{figure.simple_models}
    See Section \ref{sec:further_examples} for discussion.
    (c) The graph corresponding to
    a causal model with $A$ being the protected attribute and $Y$ some
    outcome of interest, with background variables assumed to be
    independent. (d) Expanding the model to include an intermediate
    variable indicating whether the individual is employed with two
    (latent) background variables $\textbf{Prejudiced}$ (if the person
    offering the job is prejudiced) and $\textbf{Qualifications}$ (a
    measure of the individual's qualifications). (e) A twin network
    representation of this system \citep{pearl:00} under two different
    counterfactual levels for $A$. This is created by copying nodes
    descending from $A$, which inherit unaffected parents from the
    factual world.}
\end{figure}
% \begin{figure}
%   \begin{tabular}{p{0.5\columnwidth}|p{0.5\columnwidth}}
%     \centerline{\includegraphics[width=0.5\columnwidth]{implications_fig.pdf}}&
%     \centerline{\includegraphics[width=0.5\columnwidth]{simple_models_no_q3}}(d)
%   \end{tabular}
%   \caption{\label{fig:ex1} {\bf Left:} (a) The graph corresponding to
%     a causal model with $A$ being the protected attribute and $Y$ some
%     outcome of interest, with background variables assumed to be
%     independent.  (b) Expanding the model to include an intermediate
%     variable indicating whether the individual is employed with two
%     (latent) background variables $\textbf{Prejudiced}$ (if the person
%     offering the job is prejudiced) and $\textbf{Qualifications}$ (a
%     measure of the individual's qualifications). (c) A twin network
%     representation of this system \citep{pearl:00} under two different
%     counterfactual levels for $A$. This is created by copying nodes
%     descending from $A$, which inherit unaffected parents from the
%     factual world. (d) Two causal models for different
%     real-world fair prediction scenarios.\label{figure.simple_models}
%     See Section \ref{sec:count_fair} for discussion.}
% \end{figure}

\subsection{Examples}
\label{sec:further_examples}

To provide an intuition for counterfactual fairness, we will consider
two real-world fair prediction scenarios: \textbf{insurance pricing}
and \textbf{crime prediction}. Each of these correspond to one of the
two causal graphs in Figure~\ref{figure.simple_models}(a),(b). The
Supplementary Material provides a more mathematical discussion of
these examples with more detailed insights.

\paragraph{Scenario 1: The Red Car.}
A car insurance company wishes to price insurance for car owners by
predicting their accident rate $Y$. They assume there is an unobserved
factor corresponding to aggressive driving $U$, that (a) causes
drivers to be more likely have an accident, and (b) causes individuals
to prefer red cars (the observed variable $X$). Moreover, individuals
belonging to a certain race $A$ are more likely to drive red
cars. However, these individuals are no more likely to be aggressive
or to get in accidents than any one else. We show this in
Figure~\ref{figure.simple_models}(a). Thus, using the red car feature
$X$ to predict accident rate $Y$ would seem to be an unfair prediction
because it may charge individuals of a certain race more than others,
even though no race is more likely to have an accident. Counterfactual
fairness agrees with this notion: changing $A$ while holding $U$ fixed
will also change $X$ and, consequently, $\hat Y$. Interestingly, we
can show (Supplementary Material) that in a linear model, regressing
$Y$ on $A$ and $X$ is equivalent to regressing on $U$, so
off-the-shelf regression here is counterfactually fair. Regressing $Y$
on $X$ alone obeys the FTU criterion but is not counterfactually fair,
so {\em omitting $A$ (FTU) may introduce unfairness into an otherwise
  fair world.}
\paragraph{Scenario 2: High Crime Regions.}
%%%%% ICML STUFF
%A local police precinct wants to know $Y$, whether a given house is to
%be broken into in any given day. The probability of $Y = 1$ depends on many
%unobserved factors ($U$) but also upon the neighborhood the house lies
%in ($X$). However, different ethnic groups are more likely to live in
%particular neighborhoods, and so neighborhood and break-in rates are
%often correlated with the race $A$ of the house occupier. This can be
%seen in Figure~\ref{figure.simple_models}(d) (\emph{Center}). Unlike the
%previous case, a predictor $\hat Y$ trained using $X$ and $A$ is not
%counterfactually fair. The only change from Scenario 1 is that now $Y$
%depends on $X$ as follows: $Y \!=\! \gamma U + \theta X$. Now if we
%solve for $\lambda_X,\lambda_A$ it can be shown that $\hat Y(X,a)
%\!=\! (\gamma - \frac{\alpha^2 \theta v_A}{\beta v_U})U + \alpha
%\theta a$. As this predictor depends on the values of $A$, $\hat
%Y(X,a) \!\neq\! \hat Y(X,a')$ and thus $\hat Y(X,A)$ is not
%counterfactually fair.
%%%%% END ICML STUFF
A city government wants to estimate crime rates by neighborhood to
allocate policing resources. Its analyst constructed training data by
merging (1) a registry of residents containing their neighborhood $X$
and race $A$, with (2) police records of arrests, giving each resident
a binary label with $Y = 1$ indicating a criminal arrest record.  Due
to historically segregated housing, the location $X$ depends on $A$.
Locations $X$ with more police resources have larger numbers of
arrests $Y$.  And finally, $U$ represents the totality of
socioeconomic factors and policing practices that both influence where
an individual may live and how likely they are to be arrested and
charged.  This can all be seen in
Figure~\ref{figure.simple_models}(b).

In this example, higher observed arrest rates in some neighborhoods
are due to greater policing there, not because people of different
races are any more or less likely to break the law.  The label $Y = 0$
does not mean someone has never committed a crime, but rather that
they have not been caught.  {\em If individuals in the training data
  have not already had equal opportunity, algorithms enforcing EO will
  not remedy such unfairness}.  In contrast, a counterfactually fair
approach would model differential enforcement rates using $U$ and base
predictions on this information rather than on $X$ directly.

% Unlike Scenario 1, $Y$ now
% depends on $X$ directly. If all structural equations are linear, then
% $U$ is a linear function of $A$ and $X$, and so, indirectly, a
% counterfactually fair $\hat Y$ can be expressed as a linear function
% of $A$ and $X$. However, this is different from assuming that $\hat Y$ can be
% \emph{any} linear combination of $A$ and $X$. As a matter of fact, the solution for the
% unrestricted least-squares regression of $Y$ on $A$ and $X$ cannot be
% written as a function of $U$ only, as shown in the Supplementary Material.
%Intuitively this is because the predictor $\hat Y(X,A)$ will be a
% function of not only $U$, but also of the part of $X$ that directly
% causes $Y$. This means that the change in $X$ caused by the
% counterfactual change in $A$ from $a$ to $a'$ will cause a change in
% $\hat Y(X,A)$ so that $\hat Y(X,a) \!\neq\! \hat Y(X,a')$.  In such
% a scenario, although the likelihood of a particular person having
% their house broken into does not depend upon race directly, it does
% vary which a persons race and so is not counterfactually fair.

In general, we need a multistage procedure in which we first derive
latent variables $U$, and then based on them we minimize some loss
with respect to $Y$. This is the core of the algorithm discussed next.

\subsection{Implications}
One simple but important implication of the definition of counterfactual fairness is the following:
\begin{lem}
  \label{lem:nondescend}
  Let $\mathcal G$ be the causal graph of the given model $(U, V, F)$.
  Then $\hat Y$ will be counterfactually fair if it is a function
  of the non-descendants of $A$.
\end{lem}
\begin{proof}
 Let $W$ be any non-descendant of $A$ in $\mathcal G$. Then $W_{A
   \leftarrow a}(U)$ and $W_{A \leftarrow a'}(U)$ have the same
 distribution by the three inferential steps in Section
 \ref{subsec:cmc}.  Hence, the distribution of any function $\hat Y$ of the non-descendants of $A$
 is invariant with respect to the counterfactual values of $A$.
\end{proof}

This does not exclude using a descendant $W$ of $A$ as a possible
input to $\hat Y$. However, this will only be possible in the case
where the overall dependence of $\hat Y$ on $A$ disappears, which will
not happen in general. Hence, Lemma~\ref{lem:nondescend} provides the
most straightforward way to achieve counterfactual fairness. In some
scenarios, it is desirable to define path-specific variations of
counterfactual fairness that allow for the inclusion of some descendants
of $A$, as discussed by \cite{nabi:17,kilbertus:17} and the
Supplementary Material.

\noindent{\bf Ancestral closure of protected attributes.} Suppose that
a parent of a member of $A$ is not in $A$.  Counterfactual fairness
allows for the use of it in the definition of $\hat Y$. If this seems
counterintuitive, then we argue that the fault should be at the
postulated set of protected attributes rather than with the definition
of counterfactual fairness, and that typically we should expect set
$A$ to be closed under ancestral relationships given by the causal
graph. For instance, if {\it Race} is a protected attribute, and {\it
  Mother's race} is a parent of {\it Race}, then it should also be in
$A$.

\noindent{\bf Dealing with historical biases and an existing
  fairness paradox.} The explicit difference between $\hat Y$ and $Y$
allows us to tackle historical biases. For instance, let $Y$ be an
indicator of whether a client defaults on a loan, while $\hat Y$ is
the actual decision of giving the loan. Consider the DAG $A
\rightarrow Y$, shown in Figure \ref{fig:ex1}(c) with the explicit
inclusion of set $U$ of independent background variables. $Y$ is the
objectively ideal measure for decision making, the binary indicator of
the event that the individual defaults on a loan. If $A$ is postulated
to be a protected attribute, then the predictor $\hat Y = Y = f_Y(A,
U)$ is not counterfactually fair, with the arrow $A \rightarrow Y$
being (for instance) the result of a world that punishes individuals
in a way that is out of their control. Figure \ref{fig:ex1}(d) shows a
finer-grained model, where the path is mediated by a measure of
whether the person is employed, which is itself caused by two
background factors: one representing whether the person hiring is
prejudiced, and the other the employee's qualifications. In this
world, $A$ is a cause of defaulting, even if mediated by other
variables\footnote{For example, if the function determining employment
  $f_E(A,P,Q) \equiv I_{(Q > 0, P = 0 \text{ or } A \neq a)}$ then an
  individual with sufficient qualifications and prejudiced potential
  employer may have a different counterfactual employment value for $A
  = a$ compared to $A = a'$, and a different chance of default. }. The
counterfactual fairness principle however forbids us from using $Y$:
using the twin network
\footnote{In a nutshell, this is a graph that simultaneously depicts
  ``multiple worlds'' parallel to the factual realizations. In
  this graph, all multiple worlds share the same background variables,
  but with different consequences in the remaining variables depending
  on which counterfactual assignments are provided.} of
\citet{pearl:00}, we see in Figure \ref{fig:ex1}(e) that $Y_a$ and
$Y_{a'}$ need not be identically distributed given the background
variables.
  % \footnote{We assume
  % that function $f_Y(A, U)$ is not pathological, that is, it will give
  % different outcomes for different values of $A$ other things being
  % equal. Moreover, we assume interventions in $A$ are well defined.
  % For instance, ``race'' here could be formulated as ``race
  % perception'', which can be due to, for instance, to
% racially-associated names in a C.V. or loan application.}

In contrast, any function of variables not descendants of $A$ can be
used a basis for fair decision making. This means that any variable
$\hat Y$ defined by $\hat Y = g(U)$ will be counterfactually fair for
any function $g(\cdot)$. Hence, given a causal model, the functional
defined by the function $g(\cdot)$ minimizing some predictive error
for $Y$ will satisfy the criterion, as proposed in Section
\ref{sec:algorithm}. We are essentially learning a projection of $Y$
into the space of fair decisions, removing historical biases as a
by-product.

Counterfactual fairness also provides an answer to some problems on
the incompatibility of fairness criteria. In particular, consider the
following problem raised independently by different authors (e.g.,
\cite{chouldechova:17, kleinberg:17}), illustrated below for the
binary case: ideally, we would like our predictors to obey both
Equality of Opportunity and the {\it predictive parity} criterion
defined by satisfying
\[
P(Y = 1\ |\ \hat Y = 1, A = 1) = P(Y = 1\ |\ \hat Y = 1, A = 0),
\]
\noindent as well as the corresponding equation for $\hat Y = 0$. It
has been shown that if $Y$ and $A$ are marginally associated (e.g.,
recidivism and race are associated) and $Y$ is not a deterministic
function of $\hat Y$, then the two criteria cannot be
reconciled. Counterfactual fairness throws a light in this scenario,
suggesting that both EO and predictive parity may be insufficient if
$Y$ and $A$ are associated: assuming that $A$ and $Y$ are unconfounded (as
expected for demographic attributes), this is the result of $A$ being
a cause of $Y$. By counterfactual fairness, we should {\it not} want
to use $Y$ as a basis for our decisions, instead aiming at some function
$Y_{\perp_A}$ of variables which are not caused by $A$ but are
predictive of $Y$. $\hat Y$ is defined in such a way that is an
estimate of the ``closest'' $Y_{\perp_A}$ to $Y$ according to some
preferred risk function. {\it This makes the incompatibility between EO and
  predictive parity irrelevant}, as $A$ and $Y_{\perp_A}$ will be
independent by construction given the model assumptions.

\section{Implementing Counterfactual Fairness}
\label{sec:methods}
% !TEX root=ricardo_draft.tex

As discussed in the previous Section, we need to relate $\hat Y$ to
$Y$ if the predictor is to be useful, and  we restrict
$\hat Y$ to be a (parameterized) function of the non-descendants of
$A$ in the causal graph following
Lemma~\ref{lem:nondescend}. We next introduce an algorithm, then
discuss assumptions that
can be used to express counterfactuals.

\subsection{Algorithm}
\label{sec:algorithm}

Let $\hat Y \equiv g_\theta(U, X_{\nsucc A})$ be a predictor
parameterized by $\theta$, such as a logistic regression or a neural
network, and where $X_{\nsucc A} \subseteq X$ are non-descendants of
$A$. Given a loss function $l(\cdot, \cdot)$ such as squared loss or
log-likelihood, and training data $\mathcal D \equiv \{(A^{(i)}, X^{(i)}, Y^{(i)})\}$
for $i = 1, 2, \dots, n$, we define $L(\theta) \equiv \sum_{i =
  1}^n \mathbb E[l(y^{(i)}, g_\theta(U^{(i)}, x^{(i)}_{\nsucc
    A}))\ |\ x^{(i)}, a^{(i)}] / n$ as the empirical loss to be
minimized with respect to $\theta$.  Each expectation is with respect
to random variable $U^{(i)} \sim P_{\mathcal M}(U\ |\ x^{(i)},
a^{(i)})$ where $P_{\mathcal M}(U\ |\ x, a)$ is the conditional
distribution of the background variables as given by a causal model
$\mathcal M$ that is available by assumption. If this expectation
cannot be calculated analytically, Markov chain Monte Carlo (MCMC) can
be used to approximate it as in the following algorithm.
  
\begin{algorithmic}[1]
\Procedure{FairLearning}{$\mathcal D, \mathcal M$}\Comment{Learned parameters $\hat \theta$}  
  \State For each data point $i \in \mathcal D$, sample $m$ MCMC samples
  $U_1^{(i)}, \dots, U_m^{(i)} \sim P_{\mathcal M}(U\ |\ x^{(i)},a^{(i)})$.
  \State Let $\mathcal D'$ be the augmented dataset where each point
  $(a^{(i)}, x^{(i)}, y^{(i)})$ in $\mathcal D$ is replaced with the corresponding $m$ points
  $\{(a^{(i)}, x^{(i)}, y^{(i)}, u_j^{(i)})\}$.
  \State $\hat \theta \leftarrow \mathrm{argmin}_\theta \sum_{i' \in \mathcal D'}
                                   l(y^{(i')}, g_\theta(U^{(i')}, x^{(i')}_{\nsucc A}))$.
\EndProcedure
\end{algorithmic}

At prediction time, we report $\tilde Y \equiv \mathbb E[\hat Y(U^\star,
  x^\star_{\nsucc A})\ |\ x^\star, a^\star]$ for a new data point $(a^\star,
x^\star)$.

\noindent{\bf Deconvolution perspective.} The algorithm can be
understood as a deconvolution approach that, given observables $A \cup
X$, extracts its latent sources and pipelines them into a predictive
model. We advocate that \emph{counterfactual assumptions must underlie
  all approaches that claim to extract the sources of variation of the
  data as ``fair'' latent components}. As an example,
\citet{louizos2015variational} start from the DAG $A \rightarrow X
\leftarrow U$ to extract $P(U\ |\ X, A)$. As $U$ and $A$ are not
independent given $X$ in this representation, a type of penalization
is enforced to create a posterior $P_{fair}(U\ | A, X)$ that is close
to the model posterior $P(U\ |\ A, X)$ while satisfying $P_{fair}(U\ |
A = a, X) \approx P_{fair}(U\ | A = a', X)$. But {\it this is neither
  necessary nor sufficient for counterfactual fairness}. The model for
$X$ given $A$ and $U$ must be justified by a causal mechanism, and
that being the case, $P(U\ |\ A, X)$ requires no postprocessing. As a
matter of fact, model $\mathcal M$ can be learned by penalizing
empirical dependence measures between $U$ and $pa_i$ for a given $V_i$
(e.g. \citet{mooij:09}), but this concerns $\mathcal M$ and not $\hat Y$,
and is motivated by explicit assumptions about structural equations,
as described next.

\subsection{Designing the Input Causal Model}
\label{sec:limit-guide-model}

Model $\mathcal M$ must be provided to algorithm {\sc FairLearning}.
Although this is well understood, it is worthwhile remembering that
causal models always require strong assumptions, even more so when
making counterfactual claims \cite{dawid:00}. Counterfactuals
assumptions such as structural equations are in general unfalsifiable
even if interventional data for all variables is available. This is
because there are infinitely many structural equations compatible with
the same observable distribution \cite{pearl:00}, be it observational
or interventional. Having passed testable implications, the remaining
components of a counterfactual model should be understood as
conjectures formulated according to the best of our knowledge. Such
models should be deemed provisional and prone to modifications if, for
example, new data containing measurement of variables previously
hidden contradict the current model.

We point out that we do not need to specify a fully deterministic
model, and structural equations can be relaxed as conditional
distributions. In particular, the concept of counterfactual fairness
holds under three levels of assumptions of increasing strength:

\noindent {\bf Level 1.}  Build $\hat Y$ using only the observable
non-descendants of $A$.  This only requires partial causal ordering
and no further causal assumptions, but in many problems there will be
few, if any, observables which are not descendants of protected
demographic factors.
  
\noindent {\bf Level 2.} Postulate background latent variables that
act as non-deterministic causes of observable variables, based on
explicit domain knowledge and learning algorithms\footnote{In some
  domains, it is actually common to build a model entirely around
  latent constructs with few or no observable parents nor connections
  among observed variables \citep{bol:89}.}. Information about $X$ is
passed to $\hat Y$ via $P(U\ |\ x, a)$.

\noindent {\bf Level 3.} Postulate a fully deterministic model with
latent variables. For instance, the distribution $P(V_i\ |\ pa_i)$
can be treated as an additive error model, $V_i
\!=\! f_i(pa_i) \!+\! e_i$ \citep{peters:14}. The
error term $e_i$ then becomes an input to $\hat Y$ as calculated from
the observed variables. This maximizes the information extracted by
the fair predictor $\hat Y$.

\subsection{Further Considerations on Designing the Input Causal Model}
\label{sec:pragmatic}

One might ask what we can lose by defining causal fairness measures involving
only non-counterfactual causal quantities, such as enforcing $P(\hat Y =
1\ |\ do(A = a)) = P(\hat Y = 1\ |\ do(A = a'))$ instead of our
counterfactual criterion. The reason is that the above equation is
only a constraint on an average effect. Obeying this criterion
provides no guarantees against, for example, having half of the
individuals being strongly ``negatively'' discriminated and half of
the individuals strongly ``positively'' discriminated.  We advocate
that, for fairness, society should not be satisfied in pursuing only
counterfactually-free guarantees. While one may be willing to claim
posthoc that the equation above masks no balancing effect so that
individuals receive approximately the same distribution of outcomes,
{\it that itself is just a counterfactual claim in disguise.} Our
approach is to make counterfactual assumptions explicit. When
unfairness is judged to follow only some ``pathways'' in the causal
graph (in a sense that can be made formal, see
\cite{kilbertus:17,nabi:17}), nonparametric assumptions about the
independence of counterfactuals may suffice, as discussed by
\cite{nabi:17}. In general, nonparametric assumptions may not provide
identifiable adjustments even in this case, as also discussed in our
Supplementary Material.  If competing models with different untestable
assumptions are available, there are ways of simultaneously enforcing a notion of
approximate counterfactual fairness in all of them, as introduced by
us in \cite{russell:17}. Other alternatives include exploiting
bounds on the contribution of hidden variables \cite{pearl:16,silva:16}.

Another issue is the interpretation of causal claims involving
demographic variables such as race and sex. Our view is that such
constructs are the result of translating complex events into random
variables and, despite some controversy, we consider counterproductive
to claim that e.g. race and sex cannot be causes. An idealized
intervention on some $A$ at a particular time can be seen as a notational
shortcut to express a conjunction of more specific interventions,
which may be individually doable but jointly impossible in practice.
It is the plausibility of complex, even if impossible to practically
manipulate, causal chains from $A$ to $Y$ that allows us to
claim that unfairness is real \cite{glymour:14}. Experiments for
constructs exist, such as randomizing names in job applications to
make them race-blind. They do not contradict the notion of race as a
cause, and can be interpreted as an intervention on a particular
aspect of the construct ``race,'' such as ``race perception'' (e.g. Section 4.4.4 of
\cite{pearl:16}).

\section{Illustration: Law School Success}
\label{sec:experiments}
% !TEX root=ricardo_draft.tex
% In this section we evaluate our framework for modeling fairness.
We illustrate our approach on a practical problem that requires
fairness, the \emph{prediction of success in law school}. A second
problem, \emph{understanding the contribution of race to police
  stops}, is described in the Supplementary Material. Following closely the
usual framework for assessing causal models in the machine learning
literature, the goal of this experiment is to quantify how our
algorithm behaves with finite sample sizes while assuming ground truth compatible
with a synthetic model.

\noindent {\bf Problem definition: Law school success}

% From 1991 to 1996
The Law School Admission Council
conducted a survey across 163 law
schools in the United States \cite{wightman1998lsac}. % The survey was
% designed to assess `the law school experience of minority students, as
% well as their ultimate entry into the profession'.
It contains information on 21,790 law students such as their entrance
exam scores (LSAT), their grade-point average (GPA) collected prior to
law school, and their first year average grade (FYA).
%, and following Law
%School i.e. whether students passed the final examination, the `bar
%exam' (P)).

Given this data, a school may wish to predict if an applicant will
have a high FYA.
% from information about their academic performance
% before law school.
The school would also like to make sure these
predictions are not biased by an individual's race and sex. However,
the LSAT, GPA, and FYA scores, may be biased due to social factors. % Our approach will use variables that are
% counterfactually fair for prediction.
We compare our framework with two unfair baselines: 1. \textbf{Full}:
the standard technique of using all features, including sensitive
features such as race and sex to make predictions;
2. \textbf{Unaware}: fairness through unawareness, where we do not use
race and sex as features. For comparison, we generate predictors $\hat
Y$ for all models using logistic regression.

\paragraph{Fair prediction.}
As described in Section~\ref{sec:limit-guide-model}, there are three
ways in which we can model a counterfactually fair predictor of
FYA. Level 1 uses any features which are not descendants of race and
sex for prediction. Level 2 models latent `fair' variables which are
parents of observed variables. These variables are independent of both
race and sex. Level 3 models the data using an additive error model,
and uses the independent error terms to make predictions. These models
make increasingly strong assumptions corresponding to increased
predictive power. We split the dataset 80/20 into a train/test set,
preserving label balance, to evaluate the models.

As we believe LSAT, GPA, and FYA are all biased by race and sex, we
cannot use any observed features to construct a counterfactually fair
predictor as described in Level 1. % Instead we would need to resort to a
% constant predictor% , such as the mean of FYA over the training set
% . % As this model is trivial we do not consider it. 

In Level 2, we postulate that a latent variable: a student's
\textbf{knowledge} (K), affects GPA, LSAT, and FYA scores. The causal
graph corresponding to this model is shown in
Figure~\ref{figure.law_school}, (\textbf{Level 2}). This is a
short-hand for the distributions:
\[
\begin{array}{cc}
  \mbox{GPA} \sim {\cal N}(b_{G} + w_{G}^K K + w_{G}^R R + w_{G}^S S, \sigma_{G}),&  \hspace{0.2in}
  \mbox{FYA} \sim {\cal N}(w_{F}^K K + w_{F}^R R + w_{F}^S S, 1),\\
  \mbox{LSAT} \sim \textrm{Poisson}(\exp(b_{L} + w_{L}^K K + w_{L}^R R + w_{L}^S S)),& \hspace{0.2in}
  \mbox{K} \sim {\cal N}(0,1)
\end{array}
\]
%\begin{align}
%\mbox{GPA} &\sim {\cal N}(b_{G} + w_{G}^K K + w_{G}^R R + w_{G}^S S, \sigma_{G}),
%\mbox{LSAT} &\sim \textrm{Poisson}(\exp(b_{L} + w_{L}^K K + w_{L}^R R + w_{L}^S S)) \nonumber \\
%\mbox{FYA} &\sim {\cal N}(w_{F}^K K + w_{F}^R R + w_{F}^S S, 1), 
%K &\sim {\cal N}(0,1) \nonumber
%\end{align}
% As FYA is already standardized to have mean $0$ and standard
% deviation $1$ we do not learn bias and standard deviation terms.
We perform inference on this model using an observed training set to
estimate the posterior distribution of $K$. We use the probabilistic
programming language Stan \cite{rstan} to learn $K$. We call the
predictor constructed using $K$, \textbf{Fair $K$}.

\begin{figure}[th]
  \hspace{-0.3in}
  \begin{tabular}{p{0.5\columnwidth}p{0.5\columnwidth}}
    \centerline{\includegraphics[width=0.4\columnwidth]{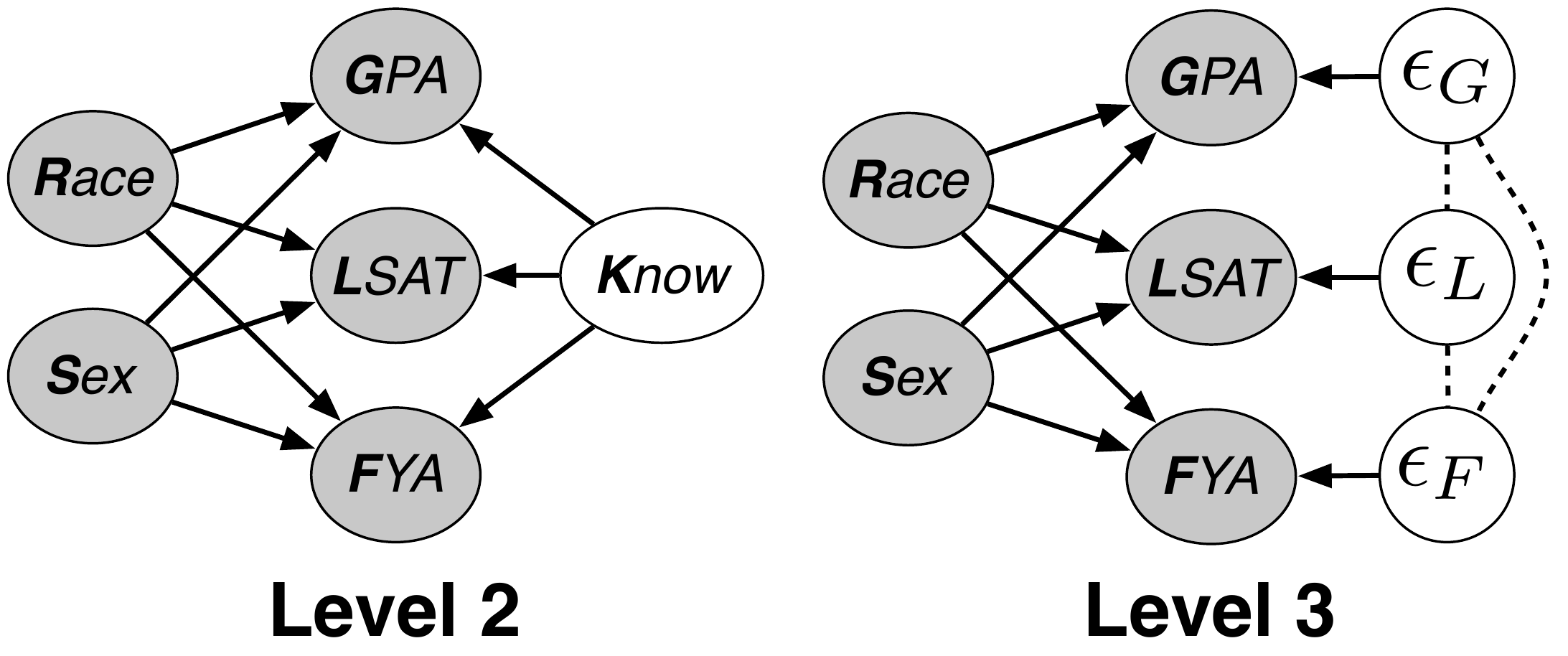}}
    &
      \centerline{\includegraphics[width=0.5\columnwidth]{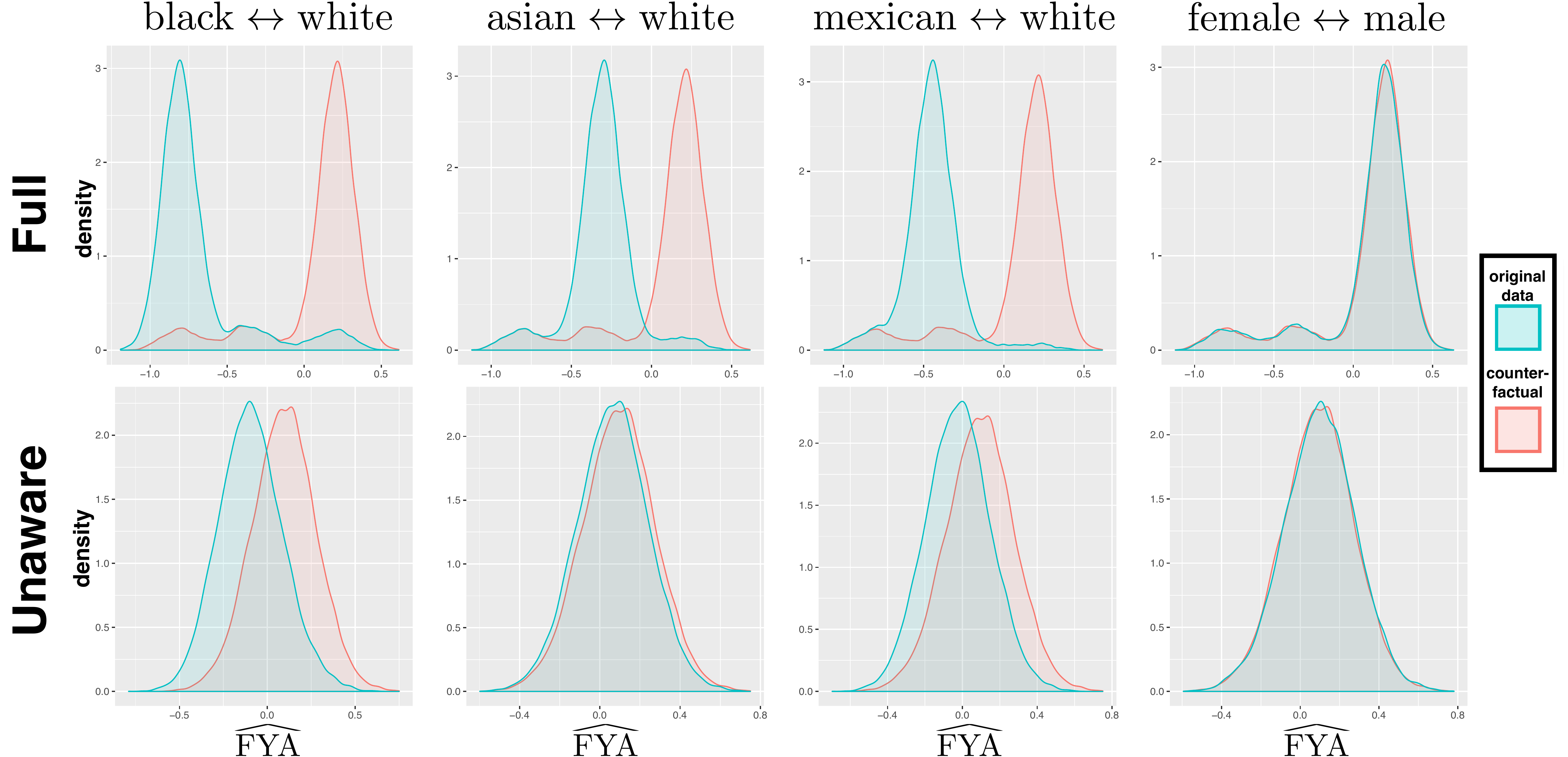}}
  \end{tabular}
  \caption{{\bf Left:} A causal model for the problem of predicting law school success fairly.\label{figure.law_school}
  {\bf Right:} Density plots of predicted $\mbox{FYA}_a$ and $\mbox{FYA}_{a'}$.\label{figure.counterfactual}
}
\end{figure}

\begin{table}
\centering
\caption{Prediction results using logistic regression. Note that we
  must sacrifice a small amount of accuracy to ensuring
  counterfactually fair prediction (Fair $K$, Fair Add), versus the
  models that use unfair features: GPA, LSAT, race, sex (Full,
  Unaware).} \label{table.pred_law}
\begin{tabular}{ccccc} 
\hline
 &  {\bf Full} & {\bf Unaware} & {\bf Fair $K$} & {\bf Fair Add} \\
\hline
RMSE & 0.873 & 0.894 & 0.929 & 0.918 \\
%\bf{Method} & %\multicolumn{2}{c}{\bf Full} & \multicolumn{2}{c}{\bf Unaware} & \multicolumn{2}{c}{\bf Fair L2} & \multicolumn{2}{c}{\bf Fair L3} \\
\hline
\end{tabular}
\end{table}

In Level 3, we model GPA, LSAT, and FYA as continuous variables with additive error terms independent of
race and sex (that may in turn be correlated with one-another). This model is shown in
Figure~\ref{figure.law_school}, (\textbf{Level 3}), and is expressed by: % the equations:
\begin{align}
\mbox{GPA} &= b_{G} + w_{G}^R R + w_{G}^S S + \epsilon_G, \;\; \epsilon_G \sim p(\epsilon_G) \nonumber \\
\mbox{LSAT} &= b_{L} + w_{L}^R R + w_{L}^S S + \epsilon_L, \;\; \epsilon_L \sim p(\epsilon_L) \nonumber \\
\mbox{FYA} &= b_{F} + w_{F}^R R + w_{F}^S S + \epsilon_F, \;\; \epsilon_F \sim p(\epsilon_F) \nonumber
\end{align}
We estimate the error terms $\epsilon_G,\epsilon_L$ by first fitting
two models that each use race and sex to individually predict GPA and
LSAT. We then compute the residuals of each model (e.g., $\epsilon_G
\!=\! \mbox{GPA} \!-\! \hat{Y}_{\scriptsize\mbox{GPA}}(R,S)$). We use
these residual estimates of $\epsilon_G,\epsilon_L$ to predict FYA. We
call this \emph{Fair Add}.

% impacts these features.

% We propose to model the law school data as shown in
% Figure~\ref{figure.law_school}. We suspect that variables race and sex
% affect student performance (e.g. GPA, LSAT, and FYA) due to factors
% such as cultural norms, which assume that individuals of a certain
% race or sex are `better suited' to be lawyers. Such beliefs could
% adversely impact students who do not fit these norms. Instead we would
% like to model the latent \emph{knowledge} (K) of a student, which also
% impacts these features. 
% We can then construct a predictor that
% predicts FYA fairly using knowledge. It is easy to show that such a predictor
% is counterfactually fair, whereas a predictor that uses features GPA and
% LSAT is not (in this case even including race and sex as
% features cannot correct this, as can be done in the linear case). The
% causal 
 %; %3. \textbf{Variational Fair Autoencoder (VFAE)} \cite{louizos2015variational}, a recent approach that works to learn a fair representation of the original data.
% compute counterfactuals for both race and sex

\paragraph{Accuracy.}
We compare the RMSE achieved by logistic regression for each of the
models on the test set in Table~\ref{table.pred_law}.  The
\textbf{Full} model achieves the lowest RMSE as it uses race and sex
to more accurately reconstruct FYA. Note that in this case, this model
is not fair even if the data was generated by one of the models shown
in Figure~\ref{figure.law_school} as it corresponds to Scenario 3. The
(also unfair) \textbf{Unaware} model still uses the unfair variables
GPA and LSAT, but because it does not use race and sex it cannot match
the RMSE of the \textbf{Full} model. As our models satisfy
counterfactual fairness, they trade off some accuracy. Our first model
\textbf{Fair $K$} uses weaker assumptions and thus the RMSE is
highest. Using the Level 3 assumptions, as in \textbf{Fair Add} we
produce a counterfactually fair model that trades
slightly stronger assumptions for lower RMSE.

\paragraph{Counterfactual fairness.}
We would like to empirically test whether the baseline methods are
counterfactually fair. To do so we will assume the true model of the
world is given by Figure~\ref{figure.law_school}, (\textbf{Level
  2}). We can fit the parameters of this model using the observed data
and evaluate counterfactual fairness by sampling from
it. Specifically, we will generate samples from the model given either
the observed race and sex, or \emph{counterfactual} race and sex
variables. We will fit models to both the original and counterfactual
sampled data and plot how the distribution of predicted FYA changes
for both baseline models. Figure~\ref{figure.counterfactual} shows
this, where each row corresponds to a baseline predictor and each
column corresponds to the counterfactual change. In each plot, the blue
distribution is density of predicted FYA for the original data and the
red distribution is this density for the counterfactual data. If a
model is counterfactually fair we would expect these distributions to
lie exactly on top of each other. Instead, we note that the
\textbf{Full} model exhibits counterfactual unfairness for all
counterfactuals except sex. We see a similar trend for the
\textbf{Unaware} model, although it is closer to being
counterfactually fair. To see why these models seem to be fair
w.r.t. to sex we can look at weights of the DAG which generates the
counterfactual data. Specifically the DAG weights from (male,female)
to GPA are ($0.93$,$1.06$) and from (male,female) to LSAT are
($1.1$,$1.1$). Thus, these models are fair w.r.t. to sex simply
because of a very weak causal link between sex and GPA/LSAT.

% here describe what we see

% maybe sample from model and check it out
%\paragraph{Model validity.}

% TODO rank top 10 students by ability or by other score in law_school.py which only considers observed features

% \begin{table}[t]
% \vspace{-2ex}
% \caption{}
% \vspace{-3ex}
% \label{table.pred_law}
% \begin{center}
% \resizebox{\columnwidth}{!}
% {
% \begin{sc}
% \footnotesize
% \begin{tabular}{c|c|c|c}
% \hline
% %\multicolumn{5}{c}{\textbf{Lower Bounds}}\\
% \hline
% & full & unaware  & fair l2 & fair l3 \\
% \hline
% RMSE & 0.873 & 0.894 & 0.929 & 0.918 \\ \hline
% \end{tabular}
% \end{sc}
% }
% \end{center}
% \vspace{-4ex}
% \end{table}

%{lr@{$\pm$}lr@{$\pm$}lr@{$\pm$}l}

%\subsection{Model criticism}
%%% Local Variables:
%%% mode: latex
%%% TeX-master: "ricardo_draft"
%%% End:

\section{Conclusion}
\label{sec:conclusion}
We have presented a new model of fairness we refer to as {\em
  counterfactual fairness}. It allows us to propose algorithms
that, rather than simply ignoring protected attributes, are able to
take into account the different social biases that may arise towards
individuals based on ethically sensitive attributes
and compensate
for these biases effectively. We experimentally contrasted our
approach with previous fairness approaches and show that our explicit
causal models capture these social biases and make clear the implicit
trade-off between prediction accuracy and fairness in an unfair
world. We propose that fairness should be regulated by explicitly
modeling the causal structure of the world. Criteria based purely on
probabilistic independence cannot satisfy this and are unable to
address \emph{how} unfairness is occurring in the task at hand. By
providing such causal tools for addressing fairness questions we hope
we can provide practitioners with customized techniques for solving a
wide array of fairness modeling problems.

\subsubsection*{Acknowledgments}

This work was supported by the Alan Turing Institute under the EPSRC
grant EP/N510129/1. CR acknowledges additional support under the EPSRC Platform Grant EP/P022529/1.
We thank Adrian Weller for insightful feedback, and the anonymous reviewers for helpful comments.

\bibliography{rbas,bibliography}

\begin{thebibliography}{39}
\providecommand{\natexlab}[1]{#1}
\providecommand{\url}[1]{\texttt{#1}}
\expandafter\ifx\csname urlstyle\endcsname\relax
  \providecommand{\doi}[1]{doi: #1}\else
  \providecommand{\doi}{doi: \begingroup \urlstyle{rm}\Url}\fi

\bibitem[Berk et~al.(2017)Berk, Heidari, Jabbari, Kearns, and Roth]{berk:17}
Berk, R., Heidari, H., Jabbari, S., Kearns, M., and Roth, A.
\newblock Fairness in criminal justice risk assessments: The state of the art.
\newblock \emph{arXiv:1703.09207v1}, 2017.

\bibitem[Bollen(1989)]{bol:89}
Bollen, K.
\newblock \emph{Structural {E}quations with {L}atent {V}ariables}.
\newblock John Wiley \& Sons, 1989.

\bibitem[Bollen \& (eds.)(1993)Bollen and (eds.)]{bollen:93}
Bollen, K. and (eds.), J.~Long.
\newblock \emph{Testing Structural Equation Models}.
\newblock SAGE Publications, 1993.

\bibitem[Bolukbasi et~al.(2016)Bolukbasi, Chang, Zou, Saligrama, and
  Kalai]{bolukbasi2016man}
Bolukbasi, Tolga, Chang, Kai-Wei, Zou, James~Y, Saligrama, Venkatesh, and
  Kalai, Adam~T.
\newblock Man is to computer programmer as woman is to homemaker? debiasing
  word embeddings.
\newblock In \emph{Advances in Neural Information Processing Systems}, pp.\
  4349--4357, 2016.

\bibitem[Brennan et~al.(2009)Brennan, Dieterich, and
  Ehret]{brennan2009evaluating}
Brennan, Tim, Dieterich, William, and Ehret, Beate.
\newblock Evaluating the predictive validity of the compas risk and needs
  assessment system.
\newblock \emph{Criminal Justice and Behavior}, 36\penalty0 (1):\penalty0
  21--40, 2009.

\bibitem[Calders \& Verwer(2010)Calders and Verwer]{calders2010three}
Calders, Toon and Verwer, Sicco.
\newblock Three naive bayes approaches for discrimination-free classification.
\newblock \emph{Data Mining and Knowledge Discovery}, 21\penalty0 (2):\penalty0
  277--292, 2010.

\bibitem[Chouldechova(2017)]{chouldechova:17}
Chouldechova, A.
\newblock Fair prediction with disparate impact: a study of bias in recidivism
  prediction instruments.
\newblock \emph{Big Data}, 2:\penalty0 153--163, 2017.

\bibitem[Dawid(2000)]{dawid:00}
Dawid, A.~P.
\newblock Causal inference without counterfactuals.
\newblock \emph{Journal of the American Statistical Association}, pp.\
  407--448, 2000.

\bibitem[DeDeo(2014)]{dedeo2014wrong}
DeDeo, Simon.
\newblock Wrong side of the tracks: Big data and protected categories.
\newblock \emph{arXiv preprint arXiv:1412.4643}, 2014.

\bibitem[Dwork et~al.(2012)Dwork, Hardt, Pitassi, Reingold, and
  Zemel]{dwork2012fairness}
Dwork, Cynthia, Hardt, Moritz, Pitassi, Toniann, Reingold, Omer, and Zemel,
  Richard.
\newblock Fairness through awareness.
\newblock In \emph{Proceedings of the 3rd Innovations in Theoretical Computer
  Science Conference}, pp.\  214--226. ACM, 2012.

\bibitem[Glymour \& Glymour(2014)Glymour and Glymour]{glymour:14}
Glymour, C. and Glymour, M.~R.
\newblock Commentary: Race and sex are causes.
\newblock \emph{Epidemiology}, 25\penalty0 (4):\penalty0 488--490, 2014.

\bibitem[Grgic-Hlaca et~al.(2016)Grgic-Hlaca, Zafar, Gummadi, and
  Weller]{grgiccase}
Grgic-Hlaca, Nina, Zafar, Muhammad~Bilal, Gummadi, Krishna~P, and Weller,
  Adrian.
\newblock The case for process fairness in learning: Feature selection for fair
  decision making.
\newblock \emph{NIPS Symposium on Machine Learning and the Law}, 2016.

\bibitem[Halpern(2016)]{halpern:16}
Halpern, J.
\newblock \emph{Actual Causality}.
\newblock MIT Press, 2016.

\bibitem[Hardt et~al.(2016)Hardt, Price, Srebro, et~al.]{hardt2016equality}
Hardt, Moritz, Price, Eric, Srebro, Nati, et~al.
\newblock Equality of opportunity in supervised learning.
\newblock In \emph{Advances in Neural Information Processing Systems}, pp.\
  3315--3323, 2016.

\bibitem[Johnson et~al.(2016)Johnson, Foster, and Stine]{johnson2016impartial}
Johnson, Kory~D, Foster, Dean~P, and Stine, Robert~A.
\newblock Impartial predictive modeling: Ensuring fairness in arbitrary models.
\newblock \emph{arXiv preprint arXiv:1608.00528}, 2016.

\bibitem[Joseph et~al.(2016)Joseph, Kearns, Morgenstern, Neel, and
  Roth]{joseph2016rawlsian}
Joseph, Matthew, Kearns, Michael, Morgenstern, Jamie, Neel, Seth, and Roth,
  Aaron.
\newblock Rawlsian fairness for machine learning.
\newblock \emph{arXiv preprint arXiv:1610.09559}, 2016.

\bibitem[Kamiran \& Calders(2009)Kamiran and Calders]{kamiran2009classifying}
Kamiran, Faisal and Calders, Toon.
\newblock Classifying without discriminating.
\newblock In \emph{Computer, Control and Communication, 2009. IC4 2009. 2nd
  International Conference on}, pp.\  1--6. IEEE, 2009.

\bibitem[Kamiran \& Calders(2012)Kamiran and Calders]{kamiran2012data}
Kamiran, Faisal and Calders, Toon.
\newblock Data preprocessing techniques for classification without
  discrimination.
\newblock \emph{Knowledge and Information Systems}, 33\penalty0 (1):\penalty0
  1--33, 2012.

\bibitem[Kamishima et~al.(2011)Kamishima, Akaho, and
  Sakuma]{kamishima2011fairness}
Kamishima, Toshihiro, Akaho, Shotaro, and Sakuma, Jun.
\newblock Fairness-aware learning through regularization approach.
\newblock In \emph{Data Mining Workshops (ICDMW), 2011 IEEE 11th International
  Conference on}, pp.\  643--650. IEEE, 2011.

\bibitem[Khandani et~al.(2010)Khandani, Kim, and Lo]{khandani2010consumer}
Khandani, Amir~E, Kim, Adlar~J, and Lo, Andrew~W.
\newblock Consumer credit-risk models via machine-learning algorithms.
\newblock \emph{Journal of Banking \& Finance}, 34\penalty0 (11):\penalty0
  2767--2787, 2010.

\bibitem[Kilbertus et~al.(2017)Kilbertus, Carulla, Parascandolo, Hardt,
  Janzing, and Sch\"{o}lkopf]{kilbertus:17}
Kilbertus, N., Carulla, M.~R., Parascandolo, G., Hardt, M., Janzing, D., and
  Sch\"{o}lkopf, B.
\newblock Avoiding discrimination through causal reasoning.
\newblock \emph{Advances in Neural Information Processing Systems 30}, 2017.

\bibitem[Kleinberg et~al.(2017)Kleinberg, Mullainathan, and
  Raghavan]{kleinberg:17}
Kleinberg, J., Mullainathan, S., and Raghavan, M.
\newblock Inherent trade-offs in the fair determination of risk scores.
\newblock \emph{Proceedings of The 8th Innovations in Theoretical Computer
  Science Conference (ITCS 2017)}, 2017.

\bibitem[Lewis(1973)]{lewis:73}
Lewis, D.
\newblock \emph{Counterfactuals}.
\newblock Harvard University Press, 1973.

\bibitem[Louizos et~al.(2015)Louizos, Swersky, Li, Welling, and
  Zemel]{louizos2015variational}
Louizos, Christos, Swersky, Kevin, Li, Yujia, Welling, Max, and Zemel, Richard.
\newblock The variational fair autoencoder.
\newblock \emph{arXiv preprint arXiv:1511.00830}, 2015.

\bibitem[Mahoney \& Mohen(2007)Mahoney and Mohen]{mahoney2007method}
Mahoney, John~F and Mohen, James~M.
\newblock Method and system for loan origination and underwriting, October~23
  2007.
\newblock US Patent 7,287,008.

\bibitem[Mooij et~al.(2009)Mooij, Janzing, Peters, and Scholkopf]{mooij:09}
Mooij, J., Janzing, D., Peters, J., and Scholkopf, B.
\newblock Regression by dependence minimization and its application to causal
  inference in additive noise models.
\newblock In \emph{Proceedings of the 26th Annual International Conference on
  Machine Learning}, pp.\  745--752, 2009.

\bibitem[Nabi \& Shpitser(2017)Nabi and Shpitser]{nabi:17}
Nabi, R. and Shpitser, I.
\newblock Fair inference on outcomes.
\newblock \emph{arXiv:1705.10378v1}, 2017.

\bibitem[Pearl(2000)]{pearl:00}
Pearl, J.
\newblock \emph{Causality: {M}odels, {R}easoning and {I}nference}.
\newblock Cambridge University Press, 2000.

\bibitem[Pearl et~al.(2016)Pearl, Glymour, and Jewell]{pearl:16}
Pearl, J., Glymour, M., and Jewell, N.
\newblock \emph{Causal Inference in Statistics: a Primer}.
\newblock Wiley, 2016.

\bibitem[Pearl(2009)]{pearl2009causal}
Pearl, Judea.
\newblock Causal inference in statistics: An overview.
\newblock \emph{Statistics Surveys}, 3:\penalty0 96--146, 2009.

\bibitem[Peters et~al.(2014)Peters, Mooij, Janzing, and
  Sch{\"o}lkopf]{peters:14}
Peters, J., Mooij, J.~M., Janzing, D., and Sch{\"o}lkopf, B.
\newblock Causal discovery with continuous additive noise models.
\newblock \emph{Journal of Machine Learning Research}, 15:\penalty0 2009--2053,
  2014.
\newblock URL \url{http://jmlr.org/papers/v15/peters14a.html}.

\bibitem[Russell et~al.(2017)Russell, Kusner, Loftus, and Silva]{russell:17}
Russell, C., Kusner, M., Loftus, J., and Silva, R.
\newblock When worlds collide: integrating different counterfactual assumptions
  in fairness.
\newblock \emph{Advances in Neural Information Processing Systems}, 31, 2017.

\bibitem[Silva \& Evans(2016)Silva and Evans]{silva:16}
Silva, R. and Evans, R.
\newblock Causal inference through a witness protection program.
\newblock \emph{Journal of Machine Learning Research}, 17\penalty0
  (56):\penalty0 1--53, 2016.

\bibitem[Stan{\ }Development{\ }Team(2016)]{rstan}
Stan{\ }Development{\ }Team.
\newblock Rstan: the r interface to stan, 2016.
\newblock R package version 2.14.1.

\bibitem[Wightman(1998)]{wightman1998lsac}
Wightman, Linda~F.
\newblock Lsac national longitudinal bar passage study. lsac research report
  series.
\newblock 1998.

\bibitem[Zafar et~al.(2015)Zafar, Valera, Rodriguez, and
  Gummadi]{zafar2015learning}
Zafar, Muhammad~Bilal, Valera, Isabel, Rodriguez, Manuel~Gomez, and Gummadi,
  Krishna~P.
\newblock Learning fair classifiers.
\newblock \emph{arXiv preprint arXiv:1507.05259}, 2015.

\bibitem[Zafar et~al.(2016)Zafar, Valera, Rodriguez, and
  Gummadi]{zafar2016fairness}
Zafar, Muhammad~Bilal, Valera, Isabel, Rodriguez, Manuel~Gomez, and Gummadi,
  Krishna~P.
\newblock Fairness beyond disparate treatment \& disparate impact: Learning
  classification without disparate mistreatment.
\newblock \emph{arXiv preprint arXiv:1610.08452}, 2016.

\bibitem[Zemel et~al.(2013)Zemel, Wu, Swersky, Pitassi, and
  Dwork]{zemel2013learning}
Zemel, Richard~S, Wu, Yu, Swersky, Kevin, Pitassi, Toniann, and Dwork, Cynthia.
\newblock Learning fair representations.
\newblock \emph{ICML (3)}, 28:\penalty0 325--333, 2013.

\bibitem[Zliobaite(2015)]{zliobaite2015survey}
Zliobaite, Indre.
\newblock A survey on measuring indirect discrimination in machine learning.
\newblock \emph{arXiv preprint arXiv:1511.00148}, 2015.

\end{thebibliography}
\bibliographystyle{icml2017}

\newpage
% !TEX root=ricardo_draft.tex
% In this section we evaluate our framework for modeling fairness.

%%%% TODO!!! ADD COMMENT THAT WE NEED Y TO LEARN THE CAUSAL MODEL!!!!!

\section*{S1 Population Level vs Individual Level Causal Effects}
\label{sec:individual}

As discussed in Section~\ref{sec:count_fair}, counterfactual fairness
is an individual-level definition. This is fundamentally different
from comparing different units that happen to share the same
“treatment” $A = a$ and coincide on the values of $X$. To see in
detail what this means, consider the following thought experiment.

Let us assess the causal effect of $A$ on $\hat Y$ by controlling
$A$ at two levels, $a$ and $a'$. In Pearl's notation, where ``$do(A = a)$''
expresses an intervention on $A$ at level $a$, we have that
\begin{equation}
\label{eq:ace}
\mathbb{E}[\hat Y\ |\ do(A = a), X = x] - \mathbb{E}[\hat Y\ |\ do(A = a'), X = x],
\end{equation}
is a measure of causal effect, sometimes called the average causal
effect (ACE). It expresses the change that is expected when we
intervene on $A$ while observing the attribute set $X = x$, under two
levels of treatment. If this effect is non-zero, $A$ is considered to
be a cause of $\hat Y$.

This raises a subtlety that needs to be addressed: in general, this
effect will be non-zero {\it even if $\hat Y$ is counterfactually
  fair}. This may sound counter-intuitive: protected attributes such
as race and gender are causes of our counterfactually fair decisions.

In fact, this is not a contradiction, as the ACE in Equation
(\ref{eq:ace}) is different from counterfactual effects. The ACE
contrasts two independent exchangeable units of the population, and it
is a perfectly valid way of performing decision analysis. However, the
value of $X = x$ is affected by different background variables
corresponding to different individuals. That is, the causal effect
(\ref{eq:ace}) contrasts two units that receive different treatments
but which happen to coincide on $X = x$. To give a synthetic example,
imagine the simple structural equation
\[
X = A + U.
\]

The ACE quantifies what happens among people with $U = x -
a$ against people with $U' = x - a'$. If, for instance, $\hat Y =
\lambda U$ for $\lambda \neq 0$, then the effect \eqref{eq:ace}
is $\lambda(a - a') \neq 0$.

Contrary to that, the counterfactual difference is zero.
That is,
\[
\mathbb{E}[\hat Y_{A \leftarrow a}(U)\ |\ A = a, X = x] -
\mathbb{E}[\hat Y_{A \leftarrow a'}(U)\ |\ A = a, X = x] =
\lambda U - \lambda U = 0.
\]

In another perspective, we can interpret the above just as if we had
{\it measured} $U$ from the beginning rather than performing
abduction. We then generate $\hat Y$ from some $g(U)$, so $U$ is the
within-unit cause of $\hat Y$ and not $A$.

If $U$ cannot be deterministically derived from $\{A = a, X = x\}$,
the reasoning is similar. By abduction, the distribution of $U$ will
typically depend on $A$, and hence so will $\hat Y$ when marginalizing
over $U$. Again, this seems to disagree with the intuition that our
predictor should be not be caused by $A$. However, this once again is
a comparison {\it across individuals}, not within an individual. 

It is this balance among $(A, X, U)$ that explains, in the examples
of Section~\ref{sec:further_examples}, why some predictors are
counterfactually fair even though they are functions of the same
variables $\{A, X\}$ used by unfair predictors: such functions must
correspond to particular ways of balancing the observables that, by
way of the causal assumptions, cancel out the effect of $A$.

\noindent {\bf More on conditioning and alternative definitions.} As discussed in
Example 4.4.4 of \citet{pearl:16}, a different proposal for
assessing fairness can be defined via the following concept:
\begin{define}[Probability of sufficiency]
  We define the probability of event $\{ A = a \}$ being a
  \emph{sufficient cause} for our
  decision $\hat Y$, contrasted against $\{ A = a' \}$, as
\begin{align}
  P(\hat Y_{A \leftarrow a'\ }(U) \neq y\ |\ X = x, A = a, \hat Y = y).
  \label{eq:sufficiency}
\end{align}
\end{define}

We can then, for instance, claim that $\hat Y$ is a fair predictor if
this probability is below some pre-specified bound for all $(x, a,
a')$. The shortcomings of this definition come from its original
motivation: to {\it explain} the behavior of an {\it existing}
decision protocol, where $\hat Y$ is the current practice and which in a
unclear way is conflated with $Y$. The implication is that if $\hat Y$
is to be designed instead of being a natural measure of existing
behaviour, then we are using $\hat Y$ itself as evidence for the
background variables $U$. This does not make sense if $\hat Y$ is
yet to be designed by us. If $\hat Y$ is to be interpreted as $Y$, then this
does not provide a clear recipe on how to build $\hat Y$: while we can
use $Y$ to learn a causal model, we cannot use it to collect training
data evidence for $U$ {\it as the outcome $Y$ will not be available to
  us at prediction time}. For this reason, we claim that while
probability of sufficiency is useful as a way of assessing an existing
decision making process, it is not as natural as counterfactual
fairness in the context of machine learning.

\noindent {\bf Approximate fairness and model validation.} The notion
of probability of sufficiency raises the question on how to define
approximate, or high probability, counterfactual fairness. This is an
important question that we address in \citep{russell:17}. Before
defining an approximation, it is important to first expose in detail
what the exact definition is, which is the goal of this paper.

We also do not address the validation of the causal assumptions used
by the input causal model of the {\sc FairLearning} algorithm in
Section \ref{sec:algorithm}. The reason is straightforward: this
validation is an entirely self-contained step of the implementation of
counterfactual fairness. An extensive literature already exists in
this topic which the practitioner can refer to (a classic account for
instance is \cite{bollen:93}), and which can be used as-is in our
context.

The experiments performed in Section \ref{sec:experiments} can be
criticized by the fact that they rely on a model that obeys our
assumptions, and ``obviously'' our approach should work better than
alternatives. This criticism is not warranted: in machine learning,
causal inference is typically assessed through simulations which
assume that the true model lies in the family covered by the
algorithm.  Algorithms, including {\sc FairLearning}, are justified in
the population sense. How different competitors behave with finite
sample sizes is the primary question to be studied in an empirical
study of a new concept, where we control for the correctness of the
assumptions. Although sensitivity analysis is important, there are
many degrees of freedom on how this can be done. Robustness issues are
better addressed by extensions focusing on approximate versions of 
counterfactual fairness. This will be covered in later work.

\noindent {\bf Stricter version.} For completeness of exposition,
notice that the definition of counterfactual fairness could be
strengthened to
\begin{align}
  \label{eq:stricter}
  P(\hat Y_{A \leftarrow a}(U) = \hat Y_{A \leftarrow a'}(U)\ |\ X = x, A = a) = 1.
\end{align}

\noindent This is different from the original definition in the case
where $\hat Y(U)$ is a random variable with a different source of
randomness for different counterfactuals (for instance, if $\hat Y$ is
given by some black-box function of $U$ with added noise that is
independent across each countefactual value of $A$). In such a
situation, the event $\{\hat Y_{A \leftarrow a}(U) = \hat Y_{A
  \leftarrow a'}(U)\}$ will itself have probability zero even if
$P(\hat Y_{A \leftarrow a}(U) = y\ |\ X = x, A = a) = P(\hat Y_{A
  \leftarrow a'}(U) = y\ |\ X = x, A = a)$ for all $y$. We do not
consider version (\ref{eq:stricter}) as in our view it does not feel
as elegant as the original, and it is also unclear whether adding an
independent source of randomness fed to $\hat Y$ would itself be
considered unfair. Moreover, if $\hat Y(U)$ is assumed to be a
deterministic function of $U$ and $X$, as in {\sc FairLearning}, then
the two definitions are the same\footnote{Notice that $\hat Y(U)$ is
  itself a random variable if $U$ is, but the source of randomness,
  $U$, is the same across all counterfactuals.}. Informally, this
stricter definition corresponds to a notion of ``almost surely
equality'' as opposed to ``equality in distribution.'' Without
assuming that $\hat Y$ is a deterministic function of $U$ and $X$,
even the stricter version does not protect us against measure zero events
where the counterfactuals are different. The definition of
counterfactual fairness concisely emphasizes that $U$ can be a random
variable, and clarifies which conditional distribution it follows. Hence, it is our
preferred way of introducing the concept even though it does not
explicit suggests whether $\hat Y(U)$ has random inputs besides $U$.

\section*{S2 Relation to Demographic Parity}
Consider the graph $A \rightarrow X \rightarrow Y$. In general, if
$\hat Y$ is a function of $X$ only, then $\hat Y$ need not obey
demographic parity, i.e.
\begin{align}
  P(\hat Y\ |\ A = a) \neq P(\hat Y\ |\ A = a'),\nonumber
\end{align}
\noindent where, since $\hat Y$ is a function of $X$, the
probabilities are obtained by marginalizing over $P(X\ |\ A = a)$ and
$P(X\ |\ A = a')$, respectively.

If we postulate a structural equation $X = \alpha A + e_X$, then given
$A$ and $X$ we can deduce $e_X$. If $\hat Y$ is a function of $e_X$
only and, by assumption, $e_X$ is marginally independent of $A$, then
$\hat Y$ is marginally independent of $A$: this follows the
interpretation given in the previous section, where we interpret $e_X$
as ``known'' despite being mathematically deduced from the
observation $(A = a, X = x)$. Therefore, the assumptions imply that
$\hat Y$ will satisfy demographic parity, and that can be falsified.
By way of contrast, if $e_X$ is not uniquely identifiable from the
structural equation and $(A, X)$, then the distribution of $\hat Y$
depends on the value of $A$ as we marginalize $e_X$, and demographic
parity will not follow. This leads to the following:
\begin{lem}
If all background variables $U' \subseteq U$ in the definition of
$\hat Y$ are determined from $A$ and $X$,
and all observable variables in the definition of $\hat Y$ are
independent of $A$ given $U'$, then $\hat Y$ satisfies demographic
parity.
% If the conditions fail, then $\hat Y$ will not satisfy demographic parity in general. 
\end{lem}
Thus, counterfactual fairness can be thought of as a counterfactual
analog of demographic parity, as present in the Red Car example further discussed
in the next section.

\section*{S3 Examples Revisited}

In Section \ref{sec:further_examples}, we discussed two examples. We
reintroduce them here briefly, add a third example, and explain some
consequences of their causal structure to the design of
counterfactually fair predictors.

\paragraph{Scenario 1: The Red Car Revisited.}
In that scenario, the structure $A \rightarrow X \leftarrow U
\rightarrow Y$ implies that $\hat Y$ should not use either $X$ or
$A$. On the other hand, it is acceptable to use $U$.  It is
interesting to realize, however, that since $U$ is related to $A$ and
$X$, there will be some association between $Y$ and $\{A, X\}$ as
discussed in Section S1. In particular, if the structural equation for
$X$ is linear, then $U$ is a linear function of $A$ and $X$, and as
such $\hat Y$ will also be a function of both $A$ and $X$. This is not
a problem, as it is still the case that the model implies that this is
merely a functional dependence that disappears by conditioning on a
postulated latent attribute $U$. Surprisingly, we must make $\hat Y$ a
indirect function of $A$ if we want a counterfactually fair predictor,
as shown in the following Lemma.

\begin{lem}
Consider a linear model with the structure in
Figure~\ref{figure.simple_models}(a).  Fitting a linear predictor to
$X$ \emph{only} is not counterfactually fair, while the same algorithm
will produce a fair predictor using \emph{both} $A$ and $X$.
\end{lem}

\begin{proof}
As in the definition, we will consider the population case, where the
joint distribution is known. Consider the case where the equations
described by the model in Figure~\ref{figure.simple_models}(a)
are deterministic and linear:
\begin{align}
X = \alpha A + \beta U, \;\;\;\; Y = \gamma U. \nonumber
\end{align}
Denote the variance of $U$ as $v_U$, the variance of $A$ as $v_A$, and
assume all coefficients are non-zero. The predictor $\hat Y(X)$
defined by least-squares regression of $Y$ on \emph{only} $X$ is given
by $\hat Y(X) \equiv \lambda X$, where $\lambda = Cov(X, Y) / Var(X)
\!=\! \beta\gamma v_U / (\alpha^2 v_A + \beta^2 v_U) \neq 0$. This 
predictor follows the concept of fairness through unawareness.

We can test whether a predictor $\hat{Y}$ is counterfactually fair
by using the procedure described in Section~\ref{subsec:cmc}:

%\begin{itemize}
{\em (i)} Compute $U$ given observations of $X,Y,A$; %, by solving for $U$.
{\em (ii)} Substitute the equations involving $A$ with an interventional value $a'$; 
% (i.e., this says: `what happens if the race of an individual were changed')
{\em (iii)} Compute the variables $X,Y$ with the interventional value
$a'$. It is clear here that $\hat Y_a(U) \!=\! \lambda(\alpha a +
\beta U) \neq \hat Y_{a'}(U)$. This predictor is not counterfactually
fair. Thus, in this case fairness through unawareness actually
perpetuates unfairness.

Consider instead doing least-squares regression of $Y$ on $X$
\emph{and} $A$. Note that $\hat Y(X,A) \equiv \lambda_X X + \lambda_A
A$ where $\lambda_X,\lambda_A$ can be derived as follows:

\begin{align}
\begin{pmatrix}
\lambda_X \\
\lambda_A
\end{pmatrix} &=
\begin{pmatrix}
Var(X) & Cov(A,X) \\
Cov(X,A) & Var(A)
\end{pmatrix}^{-1}
\begin{pmatrix}
Cov(X,Y) \\
Cov(A,Y)
\end{pmatrix} \nonumber \\
&=
%\frac{1}{(\alpha^2 v_A + \beta^2 v_U)v_A - \alpha^2 v_A^2}
\frac{1}{\beta^2 v_U v_A}
\begin{pmatrix}
v_A & -\alpha v_A \\
-\alpha v_A & \alpha^2 v_A + \beta^2 v_U
\end{pmatrix}
\begin{pmatrix}
\beta \gamma v_U \\
0
\end{pmatrix} \nonumber \\
&=
\begin{pmatrix}
\frac{\gamma}{\beta} \\
\frac{-\alpha\gamma}{\beta}
%(-\alpha\gamma) / \beta
\end{pmatrix}
\end{align}
Now imagine we have observed $A\!=\!a$. This implies that $X = \alpha
a + \beta U$ and our predictor is $\hat Y(X,a) =
\frac{\gamma}{\beta}(\alpha a + \beta U) + \frac{-\alpha\gamma}{\beta}
a = \gamma U$. Thus, if we substitute $a$ with a counterfactual $a'$
(the action step described in Section~\ref{subsec:cmc}) the predictor
$\hat Y(X,A)$ is unchanged. This is because our predictor is
constructed in such a way that any change in $X$ caused by a change in
$A$ is cancelled out by the $\lambda_A$. Thus this predictor is
counterfactually fair.
% This is known to be equivalent to first getting the residuals $R_X$ of $X$ regressed on $A$
% and regressing $Y$ on $R_X$ and $A$. Notice that $R_X \!=\! U$, and $A$ is independent of $U$
% and $Y$. Hence, $\hat Y$ will be the least squares regression of
% $Y$ on $U$, resulting on $\hat Y \!=\! \gamma U$, which is counterfactually fair.
\end{proof}

Note that if Figure~\ref{figure.simple_models}(a) is the
true model for the real world then $\hat Y(X,A)$ will also satisfy
demographic parity and equality of opportunity as $\hat Y$ will be
unaffected by $A$. 

The above lemma holds in a more general case for the structure given
in Figure~\ref{figure.simple_models}(a): any non-constant
estimator that depends only on $X$ is not counterfactually fair as
changing $A$ always alters $X$.
%We also point out that the method
%used in the proof is a special case of a general method to building a
%predictor based on information deduced about $U$ that will be
%described in the next section.

\paragraph{Scenario 2: High Crime Regions Revisited.}

The causal structure differs from the previous example by the extra
edge $X \rightarrow Y$. For illustration purposes, assume again that
the model is linear. Unlike the previous case, a predictor $\hat Y$
trained using $X$ and $A$ is not counterfactually fair. The only
change from Scenario 1 is that now $Y$ depends on $X$ as follows: $Y
\!=\! \gamma U + \theta X$. Now if we solve for $\lambda_X,\lambda_A$
it can be shown that $\hat Y(X,a) \!=\! (\gamma - \frac{\alpha^2
  \theta v_A}{\beta v_U})U + \alpha \theta a$. As this predictor
depends on the values of $A$ that are not explained by $U$, then
$\hat Y(X,a) \!\neq\! \hat Y(X,a')$ and thus $\hat Y(X,A)$ is not
counterfactually fair.

The following extra example complements the previous two examples.

\paragraph{Scenario 3: University Success.}
A university wants to know if students will be successful
post-graduation $Y$. They have information such as: grade point
average (GPA), advanced placement (AP) exams results, and other
academic features $X$. The university believes however, that an
individual's gender $A$ may influence these features and their
post-graduation success $Y$ due to social discrimination. They also
believe that independently, an individual's latent talent $U$ causes
$X$ and $Y$. The structure is similar to
Figure~\ref{figure.simple_models}(a), with the extra
edge $A \rightarrow Y$. We can again ask, is the predictor $\hat
Y(X,A)$ counterfactually fair? In this case, the different between
this and Scenario 1 is that $Y$ is a function of $U$ and $A$ as
follows: $Y \!=\! \gamma U + \eta A$. We can again solve for
$\lambda_X,\lambda_A$ and show that $\hat Y(X,a) \!=\! (\gamma -
\frac{\alpha \eta v_A}{\beta v_U})U + \eta a$. Again $\hat Y(X,A)$ is
a function of $A$ not explained by $U$, so it cannot be counterfactually fair.

% We propose to model the law school data as shown in
% Figure~\ref{figure.law_school}. We suspect that variables race and sex
% affect student performance (e.g. GPA, LSAT, and FYA) due to factors
% such as cultural norms, which assume that individuals of a certain
% race or sex are `better suited' to be lawyers. Such beliefs could
% adversely impact students who do not fit these norms. Instead we would
% like to model the latent \emph{knowledge} (K) of a student, which also
% impacts these features. 
% We can then construct a predictor that
% predicts FYA fairly using knowledge. It is easy to show that such a predictor
% is counterfactually fair, whereas a predictor that uses features GPA and
% LSAT is not (in this case even including race and sex as
% features cannot correct this, as can be done in the linear case). The
% causal 
 %; %3. \textbf{Variational Fair Autoencoder (VFAE)} \cite{louizos2015variational}, a recent approach that works to learn a fair representation of the original data.
% compute counterfactuals for both race and sex

% \begin{table}[t]
% \vspace{-2ex}
% \caption{}
% \vspace{-3ex}
% \label{table.pred_law}
% \begin{center}
% \resizebox{\columnwidth}{!}
% {
% \begin{sc}
% \footnotesize
% \begin{tabular}{c|c|c|c}
% \hline
% %\multicolumn{5}{c}{\textbf{Lower Bounds}}\\
% \hline
% & full & unaware  & fair l2 & fair l3 \\
% \hline
% RMSE & 0.873 & 0.894 & 0.929 & 0.918 \\ \hline
% \end{tabular}
% \end{sc}
% }
% \end{center}
% \vspace{-4ex}
% \end{table}
%
%{lr@{$\pm$}lr@{$\pm$}lr@{$\pm$}l}

\section*{S4 Analysis of Individual Pathways}
\label{sec:pathways}

By way of an example, consider the following adaptation of the
scenario concerning claims of gender bias in UC Berkeley's admission
process in the 1970s, commonly used a textbook example of Simpson's
Paradox. For each candidate student's application, we have $A$ as a
binary indicator of whether the applicant is female, $X$ as the choice
of course to apply for, and $Y$ a binary indicator of whether the
application was successful or not.  Let us postulate the causal graph
that includes the edges $A \rightarrow X$ and $X \rightarrow Y$ only.
We observe that $A$ and $Y$ are negatively associated, which in first
instance might suggest discrimination, as gender is commonly accepted
here as a protected attribute for college admission. However, in the
postulated model it turns out that $A$ and $Y$ are causally independent given
$X$. More specifically, women tend to choose more competitive courses
(those with higher rejection rate) than men when applying.  Our
judgment is that the higher rejection among female than male
applicants is acceptable, if the mechanism $A \rightarrow X$ is
interpreted as a choice which is under the control of the applicant.
That is, free-will overrides whatever possible cultural background conditions
that led to this discrepancy. In the framework of counterfactual
fairness, we could claim that $A$ is not a protected attribute to
begin with once we understand how the world works, and that including
$A$ in the predictor of success is irrelevant anyway once we include
$X$ in the classifier.

However, consider the situation where there is an edge $A \rightarrow
Y$, interpreted purely as the effect of discrimination after causally
controlling for $X$.  While it is now reasonable to postulate $A$ to
be a protected attribute, we can still judge that $X$ is not an unfair
outcome: there is no need to ``deconvolve'' $A$ out of $X$ to obtain
an estimate of the other causes $U_X$ in the $A \rightarrow X$
mechanism. This suggests a simple modification of the definition of
counterfactual fairness. First, given the causal graph $\mathcal G$ assumed to
encode the causal relationships in our system, define $\mathcal P_{\mathcal G_A}$ as
the set of all directed paths from $A$ to $Y$ in $\mathcal G$ which
are postulated to correspond to all unfair chains of events where $A$
causes $Y$. Let $X_{\mathcal P^c_{\mathcal G_A}} \subseteq X$ be the
subset of covariates not present in any path in $\mathcal P_{\mathcal
  G_A}$. Also, for any vector $x$, let $x_s$ represent the
corresponding subvector indexed by $S$. The corresponding uppercase
version $X_S$ is used for random vectors.

\begin{define}[(Path-dependent) counterfactual fairness]
  Predictor $\hat Y$ is {\bf (path-dependent) counterfactually fair}
  with respect to path set $\mathcal P_{\mathcal G_A}$ if under any
  context $X = x$ and $A = a$,
  \label{eq:cf_definition}
\begin{align}
  P(\hat Y_{A \leftarrow a, X_{\mathcal P^c_{\mathcal G_A}} \leftarrow\ x_{\mathcal P^c_{\mathcal G_A}}}(U) = y\ |\ X = x, A = a)  =
  \nonumber\\ 
  P(\hat Y_{A \leftarrow a', X_{\not \mathcal P^c_{\mathcal G_A}} \leftarrow\ x_{ \mathcal P^c_{\mathcal G_A}}}(U) = y\ |\ X = x, A = a), 
\end{align}
for all $y$ and for any value $a'$ attainable by $A$.
\end{define}

This notion is related to {\it controlled direct effects}
\citep{pearl:16}, where we intervene on some paths from $A$ to $Y$,
but not others. Paths in $\mathcal P_{\mathcal G_A}$ are considered
here to be the ``direct'' paths, and we condition on $X$ and $A$
similarly to the definition of probability of sufficiency
(\ref{eq:sufficiency}). This definition is the same as the original
counterfactual fairness definition for the case where $\mathcal
P^c_{\mathcal G_A} = \emptyset$. Its interpretation is analogous to
the original, indicating that for any $X_0 \in X_{\mathcal
  P^c_{\mathcal G_A}}$ we are allowed to propagate information from
the factual assigment $A = a$, along with what we learned about the
background causes $U_{X_0}$, in order to reconstruct $X_0$. The
contribution of $A$ is considered acceptable in this case and does not
need to be ``deconvolved.''  The implication is that any member of
$X_{\not \mathcal P^c_{\mathcal G_A}}$ can be included in the
definition of $\hat Y$. In the example of college applications, we are
allowed to use the choice of course $X$ even though $A$ is a
confounder for $X$ and $Y$. We are still not allowed to use $A$
directly, bypassing the background variables.

As discussed by \cite{nabi:17}, there are some counterfactual
manipulations usable in a causal definition of fairness that can be
performed by exploiting only independence constraints among the
counterfactuals: that is, without requiring the explicit description
of structural equations or other models for latent variables. A
contrast between the two approaches is left for future work, although
we stress that they are in some sense complementary: we are motivated
mostly by problems such as the one in Figure \ref{fig:ex1}(d), where
many of the mediators themselves are considered to be unfairly
affected by the protected attribute, and independence constraints
among counterfactuals alone are less likely to be useful in
identifying constraints for the fitting of a fair predictor.

%\begin{figure}[th]
%\begin{center}
%\centerline{\includegraphics[width=3in]{stop_and_frisk_model3.pdf}}
%\caption{A causal model for the stop and frisk dataset.\label{figure.stop_and_frisk}}
%\end{center}
%\end{figure}

\section*{S5 The Multifaceted Dynamics of Fairness}
\label{sec:dynamics}

One particularly interesting question was raised by one of the
reviewers: what is the effect of continuing discrimination after fair
decisions are made?  For instance, consider the case where banks
enforce a fair allocation of loans for business owners regardless of,
say, gender. This does not mean such businesses will thrive at a
balanced rate if customers continue to avoid female owned business at
a disproportionate rate for unfair reasons. Is there anything useful
that can be said about this issue from a causal perspective?

The work here proposed regards only what we can influence by changing
how machine learning-aided decision making takes place at specific
problems. It cannot change directly how society as a whole carry on
with their biases. Ironically, it may sound unfair to banks to enforce
the allocation of resources to businesses at a rate that does not
correspond to the probability of their respective success, even if the
owners of the corresponding businesses are not to be blamed by
that. One way of conciliating the different perspectives is by
modeling how a fair allocation of loans, even if it does not come
without a cost, can nevertheless increase the proportion of successful
female businesses compared to the current baseline. This change can by
itself have an indirect effect on the culture and behavior of a
society, leading to diminishing continuing discrimination by a
feedback mechanism, as in affirmative action. We believe that in the
long run isolated acts of fairness are beneficial even if we do not
have direct control on all sources of unfairness in any specific
problem.  Causal modeling can help on creating arguments about the
long run impact of individual contributions as e.g. a type of
macroeconomic assessment. There are many challenges, and we should not
pretend that precise answers can be obtained, but in theory we should
aim at educated quantitative assessments validating how a systemic
improvement in society can emerge from localized ways of addressing
fairness.

\begin{figure}[t]
\begin{center}
\centerline{\includegraphics[width=3in]{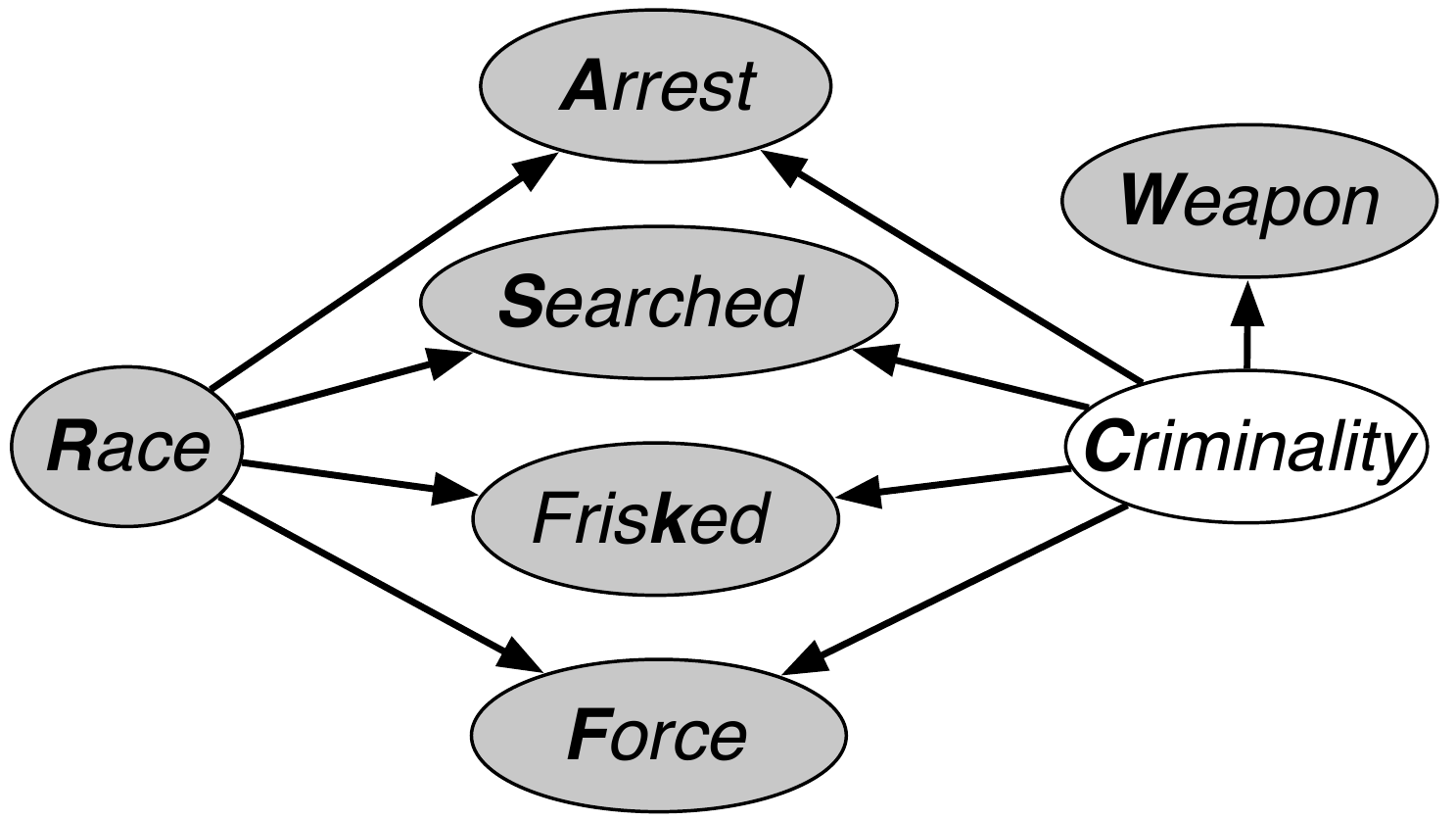}}
\caption{A causal model for the stop and frisk dataset.\label{figure.stop_and_frisk6}}
\end{center}
\end{figure}

\section*{S6 Case Study: NYC Stop-and-Frisk Data}
\label{sec:true-vs.-perceived}

Since 2002, the New York Police Department (NYPD) has recorded
information about every time a police officer has stopped someone. The
officer records information such as if the person was searched or
frisked, if a weapon was found, their appearance, whether an arrest
was made or a summons issued,  if force was used, etc. 
We consider the data collected on males stopped during 2014 which constitutes
38,609 records. We limit our analysis to looking at just males
stopped as this accounts for more than $90\%$ of the data.  We fit a
model which postulates that police interactions is caused by race and a 
single latent factor labeled \emph{Criminality} that is meant to index other aspects
of the individual that have been used by the police and which are independent of race.
We do not claim that this model has a solid theoretical basis, we use it below as an illustration
on how to carry on an analysis of  counterfactually fair decisions. We also describe a spatial analysis of
the estimated latent factors.

%\emph{Weapon} (an individual was found to be carrying a weapon),

\paragraph{Model.}
We model this stop-and-frisk data using the graph in
Figure~\ref{figure.stop_and_frisk6}. Specifically, we posit main causes
for the observations: \emph{Arrest} (if an individual was arrested),
\emph{Force} (some sort of force was used during the stop),
\emph{Frisked}, and \emph{Searched}. The first cause of these
observations is some measure of an individual's latent
\emph{Criminality}, which we do not observe. We believe that \emph{Criminality} also directly affects \emph{Weapon} (an individual was found to be carrying a weapon). For all of the features previously mentioned we believe there is an additional cause, an individual's \emph{Race} which we do observe. This factor is introduced
as we believe that these observations may be biased based on an
officer's perception of whether an individual is likely a criminal or
not, affected by an individual's \emph{Race}. Thus note that, in this model, \emph{Criminality} is counterfactually fair for the prediction of any characteristic
of the individual for problems where \emph{Race} is a protected attribute.

% \paragraph{Criminality and perception distributions.}
% After fitting this model to the data we can look at the distribution
% of \emph{Criminality} and \emph{Perception} across different races,
% shown as box plots in Figure~\ref{figure.stop_and_frisk_output}. We see that
% the median criminality for each race is nearly identical, while the
% distributions are somewhat different, demonstrating that
% \emph{Criminality} approaches demographic parity. The differences that
% due exist may be due to unobserved confounding variables that are
% affected by race or unmodeled noise in the data. On the right
% \emph{Perception} varies considerably by race with white individuals
% having the lowest perceived criminality while black and black Hispanic
% individuals have the highest.

\begin{figure*}[!t]
\begin{center}
\centerline{\includegraphics[width=\textwidth]{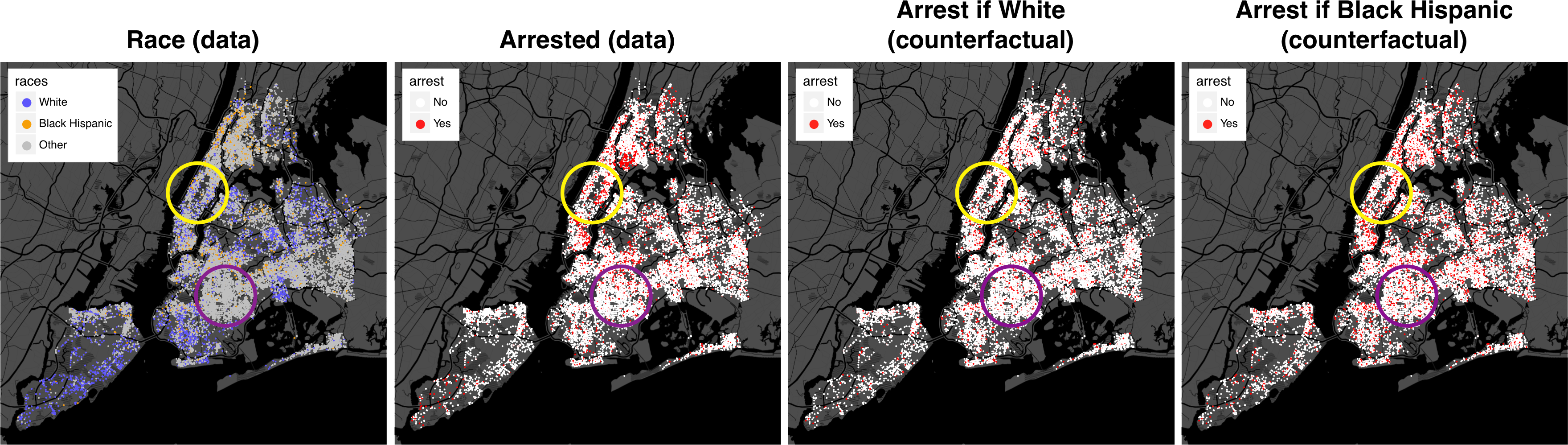}}
\caption{How race affects arrest. The above maps show how altering one's race affects whether or not they will be arrested, according to the model. The left-most plot shows the distribution of White and Black Hispanic populations in the stop-and-frisk dataset. The second plot shows the true arrests for all of the stops. Given our model we can compute whether or not every individual in the dataset would be arrest \emph{had they been white}. We show this counterfactual in the third plot. Similarly, we can compute this counterfactual if everyone had been Black Hispanic, as shown in the fourth plot.}
\label{figure.criminality2}
\end{center}
\end{figure*}

\paragraph{Visualization on a map of New York City.}
Each of the stops can be mapped to longitude and latitude points for
where the stop
occurred\footnote{https://github.com/stablemarkets/StopAndFrisk}. This allows us to visualize the distribution of two distinct populations: the stops of White and Black Hispanic individuals, shown in Figure~\ref{figure.criminality2}. We note that there are more White individuals stopped ($4492$) than Black Hispanic individuals ($2414$). However, if we look at the arrest distribution (visualized geographically in the second plot) the rate of arrest for White individuals is lower ($12.1\%$) than for Black Hispanic individuals ($19.8\%$, the highest rate for any race in the dataset). Given our model we can ask: ``If every individual had been White, would they have been arrested?''. The answer to this is in the third plot. We see that the overall number of arrests decreases (from $5659$ to $3722$). What if every individual had been Black Hispanic? The fourth plot shows an increase in the number of arrests had individuals been Black Hispanic, according to the model (from $5659$ to $6439$). The yellow and purple circles show two regions where the difference in counterfactual arrest rates is particularly striking. Thus, the model indicates that, even when everything else in the model is held constant, race has a differential affect on arrest rate under the (strong) assumptions of the model.

%\subsection{Model criticism}
%%% Local Variables:
%%% mode: latex
%%% TeX-master: "ricardo_draft"
%%% End:

\end{document}